\documentclass{article}

\usepackage[preprint]{neurips_2024}

\usepackage[utf8]{inputenc} 
\usepackage[T1]{fontenc}    
\usepackage{hyperref}       
\usepackage{url}            
\usepackage{booktabs}       
\usepackage{amsfonts}       
\usepackage{nicefrac}       
\usepackage{microtype}      
\usepackage{xcolor}         
\usepackage{amsmath,bm}
\usepackage{amssymb}
\usepackage{mathtools}
\usepackage{amsthm}
\usepackage{thmtools}
\usepackage{thm-restate}
\usepackage{graphicx}
\usepackage{subcaption}
\usepackage{algorithmic}
\usepackage{algorithm}
\usepackage[algo2e]{algorithm2e}

\newtheorem{theorem}{Theorem}[section]
\newtheorem{definition}[theorem]{Definition}
\newtheorem{proposition}[theorem]{Proposition}
\newtheorem{assumption}[theorem]{Assumption}
\newtheorem{lemma}[theorem]{Lemma}
\newtheorem{remark}[theorem]{Remark}
\newtheorem{corollary}[theorem]{Corollary}

\usepackage{systeme}
\usepackage{comment}
\usepackage[english]{babel}
\usepackage[nottoc]{tocbibind}
\usepackage{paralist}

\usepackage{amsmath,amsfonts,bm}

\def\1{\bm{1}}

\def\rmA{{\mathbf{A}}}
\def\rmB{{\mathbf{B}}}
\def\rmC{{\mathbf{C}}}
\def\rmD{{\mathbf{D}}}

\def\rmI{{\mathbf{I}}}

\def\rmM{{\mathbf{M}}}

\def\rmQ{{\mathbf{Q}}}

\def\rmS{{\mathbf{S}}}

\def\rmU{{\mathbf{U}}}
\def\rmV{{\mathbf{V}}}
\def\rmW{{\mathbf{W}}}
\def\rmX{{\mathbf{X}}}
\def\rmY{{\mathbf{Y}}}

\def\vzero{{\bm{0}}}
\def\vone{{\bm{1}}}
\def\vmu{{\bm{\mu}}}

\def\va{{\bm{a}}}
\def\vb{{\bm{b}}}

\def\vu{{\bm{u}}}
\def\vv{{\bm{v}}}
\def\vw{{\bm{w}}}
\def\vx{{\bm{x}}}
\def\vy{{\bm{y}}}
\def\vz{{\bm{z}}}

\DeclareMathAlphabet{\mathsfit}{\encodingdefault}{\sfdefault}{m}{sl}
\SetMathAlphabet{\mathsfit}{bold}{\encodingdefault}{\sfdefault}{bx}{n}

\def\gA{{\mathcal{A}}}

\def\gC{{\mathcal{C}}}
\def\gD{{\mathcal{D}}}

\def\gL{{\mathcal{L}}}

\def\gN{{\mathcal{N}}}
\def\gO{{\mathcal{O}}}

\def\gR{{\mathcal{R}}}
\def\gS{{\mathcal{S}}}

\def\sC{{\mathbb{C}}}

\def\sI{{\mathbb{I}}}

\def\sN{{\mathbb{N}}}

\def\sR{{\mathbb{R}}}

\newcommand{\E}{\mathbb{E}}

\newcommand{\indep}{\perp \!\!\! \perp}

\DeclareMathOperator{\sign}{sign}

\DeclareMathOperator{\Tr}{Tr}

\DeclareMathOperator{\asto}{\xrightarrow{\text{a.s.}}}
\DeclareMathOperator{\toind}{\xrightarrow{\gD}}

\DeclareMathOperator{\diag}{Diag}

\newcommand\ms[1]{\textcolor{red}{MS: #1}}
\newcommand\ael[1]{\textcolor{blue}{AEL: #1}}
\newcommand{\mathcolorbox}[2]{\colorbox{#1}{$\displaystyle #2$}}

\title{High-dimensional Learning with Noisy Labels}

\author{%
  Aymane El Firdoussi\thanks{Equal contribution.} \\
  Technology Innovation Institute\\
  Abu Dhabi, UAE \\
  \texttt{aymane.elfirdoussi@tii.ae} \\
  \And
  Mohamed El Amine Seddik$^*$ \\
  Technology Innovation Institute\\
  Abu Dhabi, UAE \\
  \texttt{mohamed.seddik@tii.ae} \\
}

\begin{document}

\maketitle

\begin{abstract}
    This paper provides theoretical insights into high-dimensional binary classification with class-conditional noisy labels. 
    Specifically, we study the behavior of a linear classifier with a label noisiness aware loss function, when both the dimension of data $p$ and the sample size $n$ are large and comparable.
    Relying on random matrix theory by supposing a Gaussian mixture data model, the performance of the linear classifier when $p,n\to \infty$ is shown to converge towards a limit, involving scalar statistics of the data. 
    Importantly, our findings show that the low-dimensional intuitions to handle label noise do not hold in high-dimension, in the sense that the optimal classifier in low-dimension dramatically fails in high-dimension. 
    Based on our derivations, we design an optimized method that is shown to be provably more efficient in handling noisy labels in high dimensions.
    Our theoretical conclusions are further confirmed by experiments on real datasets, where we show that our optimized approach outperforms the considered baselines.
\end{abstract}

\section{Intorduction}

Machine learning methods are usually built upon low-dimensional intuitions which do not necessarily hold when processing high-dimensional data. Numerous studies have demonstrated the effects of the curse of dimensionality, by showing that high dimensions can alter the internal functioning of various ML methods designed with low-dimensional intuitions. Classical examples include spectral methods \citep{couillet2016kernel}, empirical risk minimization frameworks \citep{el2013robust, mai2019high}, transfer \& multi-task learning \citep{tiomoko2021deciphering}, deep learning theory with the double descent phenomena \citep{nakkiran2021deep, mei2022generalization} and many other works. In all this literature, random matrix theory (RMT) played a central role in deciphering the high-dimensional effects by supposing the so-called RMT regime where both the dimension of data and the sample size are supposed to be large and comparable. We refer the reader to \citep{bai2010spectral} for a general overview on the spectral analysis of large random matrices, and to \citep{couillet2022random} for specific applications of RMT in the realm of machine learning.

In this paper, we aim at exploring the high-dimensional effects on learning with noisy labels. Based on the framework of \cite{natarajan2018cost}, who derived an unbiased classifier when faced with a binary classification problem with class-conditional noisy labels, we introduce a \textit{Labels-Perturbed Classifier (LPC)} that is essentially a Ridge classifier with parameterized labels. The introduced classifier encapsulates different variants depending on the choice of the label parameters including the unbiased method of \cite{natarajan2018cost}.
Considering a Gaussian mixture data model and supposing a high-dimensional regime, we conduct an RMT analysis of LPC by characterizing the distribution of its decision function and deriving its theoretical test performance in terms of both accuracy and risk. Our analysis allows us to gain insight when learning with noisy labels, and more importantly design an optimized classifier that surprisingly outperforms the unbiased classifier of \cite{natarajan2018cost} in high dimensions, even approaching the performance of an oracle classifier that is trained with the correct labels. Through this analysis, we demonstrate again that methods designed with low-dimensional intuitions can dramatically fail in high-dimensions, and careful refinements are needed to design more robust and interpretable methods. Our theoretical findings are also validated on real data where we show consistent improvements under high label noise. 

The remainder of the paper is organized as follows. Section \ref{sec_related_work} presents related work in the realm of learning with noisy labels. Our setting and main assumptions along with essential RMT notions are presented in Section \ref{sec_setting}. The main results brought by this paper are deferred to Section \ref{sec_main_results}. In Section \ref{sec_experiments} we conduct experiments to validate our findings on real data. Finally, Section \ref{sec_conclusion} concludes the paper and discusses future extensions. All our proofs are deferred to the Appendix.
\section{Related work}\label{sec_related_work}

Numerous studies have been conducted to investigate supervised learning under noisy labels, spanning both theoretical and empirical approaches. These studies range from learning theory and statistical perspectives to practical implementations using neural networks and deep learning techniques. 

Key contributions in this field include: \textit{Bayesian Approaches:} \cite{graepel2000kernel} conducted a Bayesian study on learning with noisy labels. \cite{lawrence2001estimating} estimated noise levels in kernel-based learning a work that was later extended by \cite{li2007classification}, who incorporated a probabilistic noise model into the Kernel Fisher discriminant and relaxed distribution assumptions. \textit{Robust Optimization Approaches:} \cite{freund2009more} proposed a robust boosting algorithm using a non-convex potential, which demonstrated empirical resilience against random label noise. \cite{jiang2001some} provided a survey of theoretical results on boosting with noisy labels. \textit{Model-Specific Robustness:} \cite{biggio2011support} explored the robustness of SVMs under adversarial label noise and proposed a kernel matrix correction method to enhance robustness. \textit{Algorithmic Innovations:} Several noise-tolerant versions of the perceptron algorithm have been developed, including Passive-aggressive algorithms \citep{crammer2006online}, Confidence-weighted learning \citep{dredze2008confidence}, AROW \citep{crammer2009adaptive}, and NHERD algorithm \citep{crammer2010learning}. \textit{Deep Learning Approaches:} Recent works have utilized deep learning techniques to address noisy labels. For example, \cite{li2020dividemix} introduced Dividemix, a semi-supervised learning algorithm for learning with noisy labels. \cite{ma2018dimensionality} studied the generalization behavior of deep neural networks (DNNs) for noisy labels in terms of intrinsic dimensionality, proposing a Dimensionality-Driven Learning (D2L) strategy to avoid overfitting. \cite{tanaka2018joint} addressed noisy labels in computer vision contexts, while \cite{karimi2020deep} applied these techniques to medical imaging.

Our work is closely related to the studies in \citep{natarajan2013learning, natarajan2018cost}, which consider adaptive loss functions and assume the prior knowledge of the noise rates. \cite{scott2013classification} do not make this assumption and model the true distribution as satisfying a mutual irreducibility property, then estimating mixture proportions by maximal denoising of noisy distributions. \cite{manwani2013noise} investigated the impact of the loss function on noise tolerance, showing that empirical risk minimization under the 0-1 loss has robust properties, while the squared loss is noise-tolerant only under uniform noise. For a comprehensive overview of the field, readers can refer to the survey by \cite{song2022learning} on learning with noisy labels.

\section{Problem setting and Background}\label{sec_setting}
\subsection{Binary classification with noisy labels}
We consider that we are given a sequence of $n$ i.i.d  $p$-dimensional training data $\vx_1, ..., \vx_n \in \sR^p$ with corresponding correct labels $y_1, ..., y_n=\pm 1$. We consider a noisy label setting where the true labels $y_i$'s are flipped randomly, yielding a noisy dataset $(\vx_i,\tilde y_i)_{i\in [n]}$ such that
\begin{align*}
    \mathbb P(\tilde y_i=-1\mid y_i=+1)=\varepsilon_+,\quad \mathbb P(\tilde y_i=+1\mid y_i=-1)=\varepsilon_-,\quad \text{with}\quad \varepsilon_+ + \varepsilon_- < 1.
\end{align*}
We suppose that $\vx_i$ is sampled from a Gaussian mixture of two clusters $\gC_1$ and $\gC_2$, i.e., for $a\in [2]$:
\begin{align}\label{eq_data_model_isotropic}
    \vx_i \in \gC_a \quad \Leftrightarrow \quad \begin{cases}
        \vx_i = \vmu_a + \vz_i, \quad \vz_i \sim \gN(\vzero, \rmI_p), \\
        y_i = (-1)^a.
    \end{cases}
\end{align}
For convenience and without loss of generality, we further assume that $\vmu_a = (-1)^a \vmu$ for some vector $\vmu \in \sR^p$. This setting can be recovered by subtracting $\frac{\vmu_1 + \vmu_2}{2}$ from each data point, as such $\vmu = \frac{ \vmu_2 - \vmu_1 }{2}$ and therefore the SNR $\Vert \vmu \Vert$ controls the difficulty of the classification problem, in the sense that large values of $\Vert \vmu \Vert$ yield a simple classification problem whereas when $\Vert \vmu \Vert \to 0$, the classification becomes impossible.

\begin{remark}[On the data model]\label{remark_distribution}
    Note that the above data assumption can be relaxed to considering $\vx_i = \vmu_a + \rmC_a^{\frac12} \vz_i$ where $\rmC_a$ is some semi-definite covariance matrix and $\vz_i$ are random vectors with i.i.d entries of mean $0$, variance $1$ and bounded fourth order moment. In fact, in the high-dimensional regime when $n,p\to \infty$, the asymptotic performance of the classifier considered subsequently is universal in the sense that it depends only on the statistical means and covariances of the data \citep{louart2018concentration, seddik2020random, dandi2024universality}. However, such a general setting comes at the expense of more complex formulas, making the above isotropic assumption more convenient for readability and better interpretation of our findings. We provide a more general result of our main result (Theorem \ref{thm_main}) by considering arbitrary covariance matrices in the Appendix (Theorem \ref{thm_general_covariance}).
\end{remark}

\paragraph{Naive approach}
Given the noisy dataset $(\vx_i,\tilde y_i)_{i\in [n]}$ as per \eqref{eq_data_model_isotropic}, a naive learning approach would consist in ignoring the noisiness of the labels and training a given classifier, such as a Ridge classifier which consists of minimizing the following:
\begin{align*}
    \mathcal L_0(\vw) = \frac1n \sum_{i=1}^n (\vw^\intercal \vx_i-\tilde y_i)^2  + \gamma\| \vw \|^2,
\end{align*}
where $\gamma\geq 0$ is a regularization parameter. Therefore, the solution for the naive classifier is given by:
\begin{align*}
    \vw_0 = \frac1n \rmQ(\gamma) \rmX \tilde \vy, \quad \rmQ(z) = \left( \frac1n \rmX \rmX^\top + z \rmI_p \right)^{-1},
\end{align*}
where $\rmX = \left[ \vx_1, \ldots, \vx_n \right] \in \sR^{p\times n}$ and $\tilde \vy = (\tilde y_1, \ldots, \tilde y_n)^\top \in \sR^n$. 

\paragraph{Improved approach} 
\cite{natarajan2018cost} proposed an unbiased approach which takes into account the noisiness of the labels. Specifically, given any bounded loss function $\ell(s,y)$, their approach consists in considering:
\begin{align*}
    \tilde\ell(s, y)\equiv \frac{(1-\varepsilon_{-y})\ell(s, y) - \varepsilon_{y}\ell(s, -y)}{1-\varepsilon_+ -\varepsilon_-}.
\end{align*}
The main intuition behind this proposition is that this loss has the nice property of being an unbiased estimator of the loss $\ell(s,y)$ on the correct dataset $(\vx_i, y_i)_{i\in [n] }$, since it satisfies for any $s,y$:
\begin{align*}
    \mathbb{E}_{\tilde y}[\tilde \ell(s,\tilde y)] = \ell(s,y).
\end{align*}
In the remainder, we consider the following loss which introduces scalar parameters $\rho_\pm$, to be optimized, rather than $\varepsilon_\pm$:
\begin{align}
\label{lpc-loss}
    \tilde\ell(s, y, \rho)\equiv \frac{(1-\rho_{-y})\ell(s, y) - \rho_{y}\ell(s, -y)}{1-\rho_+ -\rho_-},
\end{align}
Hence, for $\ell(s,y)=(s-y)^2$ and supposing a linear classifier $s(\vx)=\vw^\top \vx$, the empirical loss with $\tilde \ell$ reads as:
\begin{align*}
    \mathcal L_\rho(\vw) = \frac1n \sum_{i=1}^n \frac{(1-\rho_{-\tilde y_i})(\vw^\top \vx_i-\tilde y_i)^2 - \rho_{\tilde y_i}(\vw^\top \vx_i+\tilde y_i)^2}{1-\rho_+ -\rho_-} + \gamma\| \vw \|^2.
\end{align*}
The solution of which defines our \emph{Labels-Perturbed Classifier (LPC)} as follows:
\begin{align} \label{w_imp}
     \vw_\rho = \frac1n \rmQ(\gamma) \rmX \rmD_\rho\tilde \vy, \quad \rmQ(z) = \left(\frac1n \rmX \rmX^\top + z \rmI_p\right)^{-1},
\end{align}
where $\rmD_\rho$ is a diagonal matrix defined as $\rmD_\rho = \diag\left(\frac{1-\rho_{-\tilde y_i} + \rho_{\tilde y_i}}{1-\rho_+ -\rho_-}\mid i\in[n]\right)\in \mathcal{D}_n$. In the remainder, we will study the performance of $\vw_\rho$ which encapsulates the following cases:
\begin{itemize}
    \item \textit{Naive Classifier:} which corresponds to $\rho_\pm = 0$.
    \item \textit{Unbiased Classifier:} by taking $\rho_\pm = \varepsilon_\pm$ as introduced by \cite{natarajan2018cost}.
    \item \textit{Optimized Classifier:} by optimizing $\rho_\pm$ to maximize the theoretical test accuracy.
    \item \textit{Oracle Classifier:} which corresponds to training on the correct labels, i.e., $\rho_\pm = \varepsilon_\pm = 0$.
\end{itemize}

We aim to characterize the asymptotic performance (i.e., test accuracy and risk) of LPC in the high-dimensional regime where both the sample size $n$ and the data dimension $p$ grow large at a comparable rate, which corresponds to the classical random matrix theory (RMT) regime. Specifically, our analysis confirms that the \textit{unbiased} classifier outperforms the \textit{naive} classifier in a low-dimensional regime, i.e., when $n\gg p$. In contrast, when considering the RMT regime, we show that the \textit{unbiased} classifier becomes sub-optimal and we provide an \textit{optimized} classifier that consists of maximizing the derived test accuracy w.r.t the scalars $\rho_\pm$ yielding a closed-form solution. This sheds light on the fact that low-dimensional intuitions do not necessarily hold for high dimensions and careful refinements should be considered to enhance the performance of simple algorithms in these settings. Moreover, and of independent interest, our analysis allows us to design a method to estimate the rates $\varepsilon_\pm$ which is a key step of our approach and the \textit{unbiased} classifier \citep{natarajan2018cost}.

\subsection{RMT Background}
In mathematical terms, the understanding of the asymptotic performance of the classifier $\vw_\rho$ boils down to the characterization of the statistical behavior of the \textit{resolvent matrix} $\rmQ(z)$ introduced in \eqref{w_imp}. In the following, we will recall some important notions and results from random matrix theory which will be at the heart of our analysis. We start by defining the main object which is the resolvent matrix.
\begin{definition}[Resolvent]
    For a symmetric matrix $\rmM \in \sR^{p \times p}$, the resolvent $\rmQ_{M}(z)$ of $\rmM$ is defined for $z \in \sC \backslash  \gS (\rmM)$ as:
    \begin{equation*}
        \rmQ_{M}(z) = (\rmM - z \rmI_p)^{-1},
    \end{equation*}
    where $\gS (\rmM)$ is the set of eigenvalues or spectrum of $\rmM$.
\end{definition}
The matrix $\rmQ_{M}(z)$ will often be denoted $\rmQ(z)$ or $\rmQ$ when there is no ambiguity.
In fact, the study of the asymptotic performance of $\vw_\rho$ involves the estimation of linear forms of the resolvent $\rmQ$ in \eqref{w_imp}, such as $\frac1n\Tr\rmQ$ and $\va^\top \rmQ \vb$ with $\va,\vb\in \sR^p$ of bounded Euclidean norms. Therefore, the notion of a \textit{deterministic equivalent} \citep{hachem2007deterministic} is crucial as it allows the design of a deterministic matrix, having (in probability or almost surely) asymptotically the same \textit{scalar observations} as the random ones in the sense of \textit{linear forms}. 
A rigorous definition is provided below.

\begin{definition}[Deterministic equivalent \citep{hachem2007deterministic}]
    We say that $\bar \rmQ \in \sR^{p \times p}$ is a deterministic equivalent for the random resolvent matrix $\rmQ \in \sR^{p \times p}$ if, for any bounded linear form $u:\sR^{p\times p} \to \sR$, we have that, as $p\to \infty$:
    \begin{align*}
        u(\rmQ) \asto u(\bar \rmQ),
    \end{align*}
    where the convergence is in the \textit{almost sure} sense.
\end{definition}

In particular, a deterministic equivalent for the resolvent $\rmQ(z)$ defined in \eqref{w_imp} is given by the following Lemma, a result that is brought from \citep{louart2018concentration}.

\begin{lemma}[Deterministic equivalent of the resolvent]
\label{DEQ}
Under the high-dimensional regime, when $p,n\to \infty$ with $\frac{p}{n} \to \eta \in (0, \infty)$ and assuming $\Vert \vmu \Vert = \gO(1)$. A deterministic equivalent for $\rmQ\equiv\rmQ(\gamma)$ as defined in \eqref{w_imp} is given by:
\begin{align*}
    \bar \rmQ = \left ( \frac{\vmu \vmu^\top + \rmI_p}{ 1 + \delta} + \gamma \rmI_p \right)^{-1}, \quad \delta = \frac{1}{n} \Tr \bar \rmQ = \frac{\eta - \gamma - 1 + \sqrt{(\eta - \gamma - 1)^2 + 4 \eta \gamma}}{2 \gamma}.
\end{align*}
\end{lemma}

In a low-dimensional setting, i.e. when $p$ being fixed while $n\to \infty$, the resolvent $\rmQ$ converges almost surely to $\left ( \vmu \vmu^\top + (1 + \gamma) \rmI_p \right)^{-1}$ which is also covered by Lemma \ref{DEQ} as $\delta\to 0$ in this setting. However, when both $p$ and $n$ are large and comparable, the data dimension induces a bias which is captured by the quantity $\delta$ as it becomes $\gO(1)$ in the RMT regime. We will highlight in the following that this bias alters the behavior of the classifier $\vw_\rho$ in high dimensions, in particular, making the \textit{unbiased} classifier $\vw_\varepsilon$ introduced by \cite{natarajan2018cost} unexpectedly sub-optimal when learning with noisy labels in high-dimensions. 

\section{Main Results}\label{sec_main_results}
\subsection{Asymptotic Behavior of the Labels-Perturbed Classifier (LPC)}

We are now in place to present our main technical result which describes the asymptotic behavior of LPC as defined in \eqref{w_imp}. Specifically, we provide our results under the following growth rate assumptions.

\begin{assumption}[Growth Rates]\label{assum_growth_rate}
    Suppose that as $p,n\to \infty$:
    \begin{center}
        1) $\frac{p}{n} \to \eta \in (0, \infty)$,\hspace{1cm}
        2) $\frac{n_a}{n} \to \pi_a\in (0, 1)$, \hspace{1cm}
        3) $\Vert \vmu \Vert = \gO(1)$,
    \end{center}
    where $n_a$ denotes the cardinality of the class $\gC_a$ for $a\in [2
    ]$.
\end{assumption}

We emphasize that the condition $\Vert \vmu \Vert = \gO(1)$ reflects the fact that as the dimension $p$ grows large, the classification problem is neither impossible nor trivial making this assumption reasonable in the considered high-dimensional regime. We refer the reader to \citep{couillet2016kernel} for a more general formulation of this assumption under a $k$-class Gaussian mixture model.

Further, define the following quantities which will be used subsequently:
\begin{align}\label{eq_quantities}
    \lambda_- = \frac{1 - \rho_+ + \rho_-}{1 - \rho_+ - \rho_-}, \,\, \lambda_+ = \frac{1 - \rho_- + \rho_+}{1 - \rho_+ - \rho_-}, \,\, \beta = \frac{1}{1 - \rho_+ - \rho_-}, \,\, h=1 -  \frac{\eta}{(1 + \gamma (1 + \delta)^2)}.
\end{align}
Our main result is therefore given by the following theorem.
\begin{theorem}[Gaussianity of LPC]\label{thm_main}
    Let $\vw_{\rho}$ be the LPC as defined in \eqref{w_imp} and suppose that Assumption \ref{assum_growth_rate} holds. The decision function $\vw_\rho^\top \vx$, on some test sample $\vx\in \gC_a$ independent from $\rmX$, satisfies:
    \begin{align*}
    \vw_{\rho}^\top \vx \,\, \toind  \,\, \gN\left( (-1)^a m_\rho,\,  \nu_\rho - m_\rho^2 \right),
\end{align*}
where:
\begin{align*}
    m_\rho &= \frac{ \pi_1(\lambda_- - 2\beta \varepsilon_-) + \pi_2 (\lambda_+ - 2 \beta \varepsilon_+) }{\Vert \vmu \Vert^2 + 1 + \gamma(1 + \delta)} \Vert \vmu \Vert^2, \\
    \nu_\rho &= \frac{(\pi_1(2 \beta \varepsilon_- - \lambda_-) + \pi_2(2 \beta \varepsilon_+ - \lambda_+))^2}{h (\Vert \vmu \Vert^2 + 1 + \gamma (1 + \delta))} \left ( \frac{\Vert \vmu \Vert^2 + 1}{\Vert \vmu \Vert^2 + 1 + \gamma (1 + \delta)} - 2(1 - h) \right) \Vert \vmu \Vert^2 \\
    +& \frac{(1 - h)}{h} \left( \pi_1 (4 \beta^2 \varepsilon_-(\rho_+ - \rho_-) + \lambda_-^2 ) + \pi_2(4 \beta^2 \varepsilon_+(\rho_- - \rho_+) + \lambda_+^2) \right).
\end{align*}
\end{theorem}

\begin{figure}[t!]
\centering
\includegraphics[width = \textwidth]{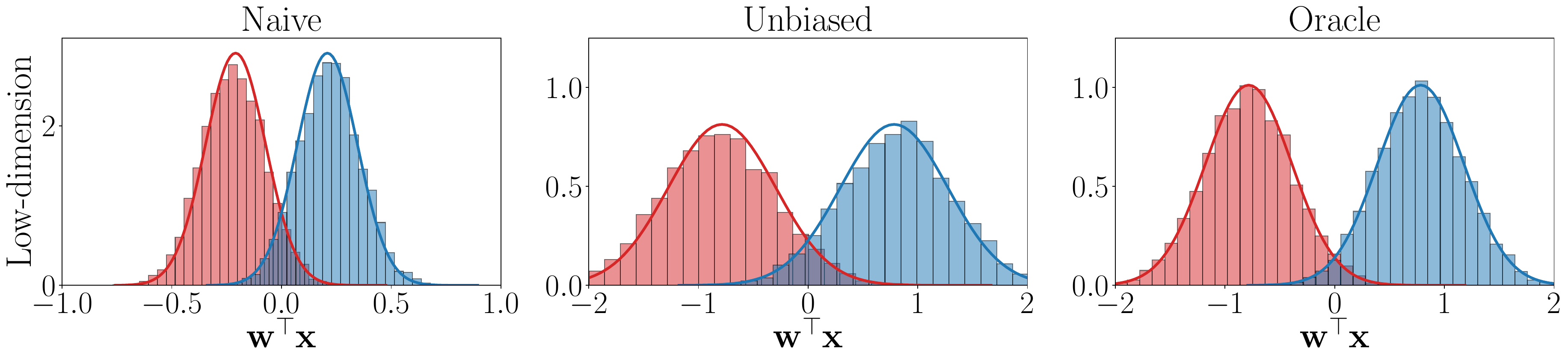}
\includegraphics[width=\textwidth]{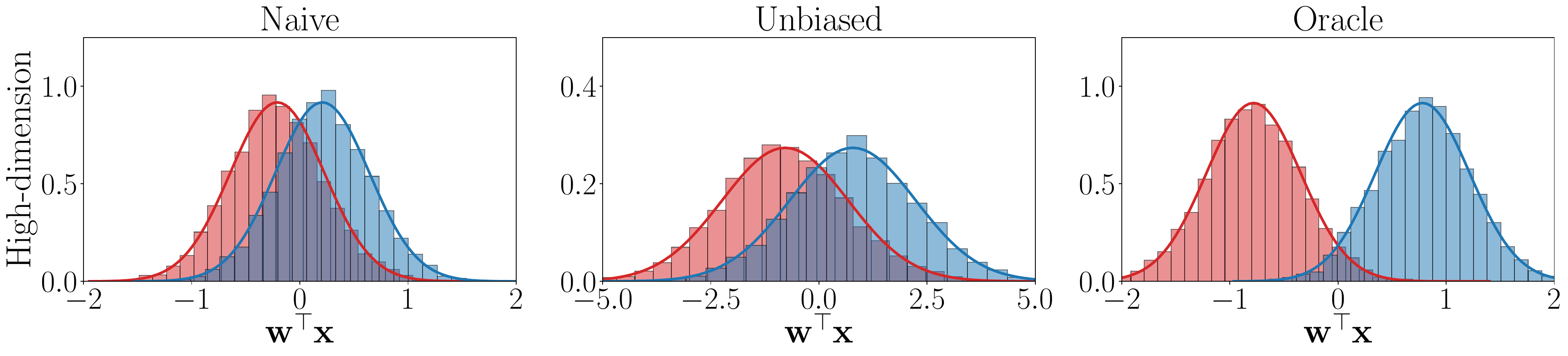}
\includegraphics[width=.8\textwidth]{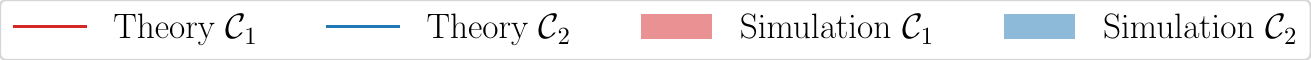}
\caption{Distribution of the decision function $\vw_\rho^\top \vx$ of different variants of LPC for $n = 5000$, $\pi_1 = \frac13 $, $\varepsilon_+ = 0.4$, $\varepsilon_- = 0.3$, $ \Vert \vmu \Vert = 2 $, $\gamma = 0.1$, $p = 50$ (first row) and $p = 1000$ (second row). The theoretical Gaussian distributions are predicted as per Theorem \ref{thm_main}. Note that the variance of the decision function for the \textit{unbiased} classifier increases with the dimension yielding poor accuracy.}
\label{fig:distribution-1}
\end{figure}

In a nutshell, Theorem \ref{thm_main} states that LPC is asymptotically equivalent to the thresholding of two monovariate Gaussian random variables with respective means $-m_\rho$ and $m_\rho$ and second moment $\nu_\rho$, where these statistics express in terms of the different parameters in our setting. Essentially, Theorem \ref{thm_main} allows us to draw interpretations on the behavior of the different classifiers described earlier. First, let us start by defining the statistics for the \textit{oracle} classifier which corresponds to setting $\rho_\pm=\varepsilon_\pm=0$, yielding:
\begin{align}
    {m_{\text{oracle}}} &= \frac{ \Vert \vmu \Vert^2 }{\Vert \vmu \Vert^2 + 1 + \gamma(1 + \delta)}, \quad \nu_{\text{oracle}} = \kappa  + {\color{red}\frac{1 - h}{h}},\\
    \text{where}\quad{\kappa} &= \frac{1}{h (\Vert \vmu \Vert^2 + 1 + \gamma (1 + \delta))} \left ( \frac{\Vert \vmu \Vert^2 + 1}{\Vert \vmu \Vert^2 + 1 + \gamma (1 + \delta)} - 2(1 - h) \right) \Vert \vmu \Vert^2.
\end{align}
Therefore, the statistics of the decision functions for the \textit{naive} ($\rho_\pm=0$) and \textit{unbiased} ($\rho_\pm=\varepsilon_\pm$) classifiers are expressed respectively as follows:
\begin{align*}
    \textit{Naive}\,
    &\begin{cases}
        m_{\text{naive}} &= (1 - 2(\pi_1 \varepsilon_- + \pi_2 \varepsilon_+ )) \cdot {m_{\text{oracle}}},\\
    \nu_{\text{naive}} &= (1 - 2(\pi_1 \varepsilon_- + \pi_2 \varepsilon_+ ))^2 \cdot {\kappa}  + {\color{red}\frac{1 - h}{h}}.
    \end{cases}\\
    \textit{Unbiased}\,
    &\begin{cases}
        m_{\text{unbiased}} &= {m_{\text{oracle}}},\\
    \nu_{\text{unbiased}} &=  {\kappa}  + {\color{red}\frac{1 - h}{h}} \left (  \pi_1(4 \beta^2 \varepsilon_-(\varepsilon_+ - \varepsilon_-) + \lambda_-^2) + \pi_2 (4 \beta^2 \varepsilon_+(\varepsilon_- - \varepsilon_+) + \lambda_+^2 )  \right).
    \end{cases}
\end{align*}

From these quantities, we can explain the behavior of the different classifiers in the low-dimensional versus high-dimensional regimes. In fact, when $n\gg p$ the dimensions ratio $\eta \to 0$ implies that $h\to 1$ as per \eqref{eq_quantities}. Therefore, in the low-dimensional setting, the \textit{unbiased} classifier statistics match those of the \textit{oracle} as expected. However, in the high-dimensional regime, i.e., when $h\neq 1$, while the \textit{unbiased} classifier remains unbiased, the second moment gets amplified due to label noise, resulting in a larger variance compared with the \textit{oracle} classifier. Indeed, we have:
\begin{align*}
    & m_{\text{unbiased}} - m_{\text{oracle}} = 0, \\
    &\nu_{\text{unbiased}} - \nu_{\text{oracle}} = {\color{red}\frac{1 - h}{h}} \left (  \pi_1(4 \beta^2 \varepsilon_-(\varepsilon_+ - \varepsilon_-) + \lambda_-^2) + \pi_2 (4 \beta^2 \varepsilon_+(\varepsilon_- - \varepsilon_+) + \lambda_+^2 ) - 1  \right) \neq 0.
\end{align*}

This behavior is highlighted in Figure \ref{fig:distribution-1} which depicts the histogram of the decision function for the different classifiers along with the theoretical Gaussian distributions as per Theorem \ref{thm_main}, in both the low-dimensional and high-dimensional settings. Moreover, having characterized the distribution of the decision function of $\vw_\rho$ allows us to estimate its generalization performance such as the test accuracy $\gA_{\text{test}}$ and test risk $\gR_{\text{test}}$ which are defined respectively, for a test set $\left( \vx_i^{\text{test}}, y_i^{\text{test}} \right)_{i\in [n_{\text{test}}]}$ independent from the training set with $y_i^{\text{test}}$ being correct labels, as follows:
\begin{align}
    \gA_{\text{test}} = \frac{1}{n_{\text{test}}} \sum_{i=1}^{n_{\text{test}}} \vone\{ \sign(\vw_\rho^\top \vx_i^{\text{test}} ) = y_i^{\text{test}}  \}, \quad \gR_{\text{test}} = \frac{1}{n_{\text{test}}} \sum_{i=1}^{n_{\text{test}}} \left( \vw_\rho^\top \vx_i^{\text{test}}  - y_i^{\text{test}} \right)^2.
\end{align}
Essentially, we have the following proposition under Assumption \ref{assum_growth_rate} and taking $n_{\text{test}}\to \infty$.

\begin{figure}[t!]
    \centering
    \includegraphics[width = \textwidth]{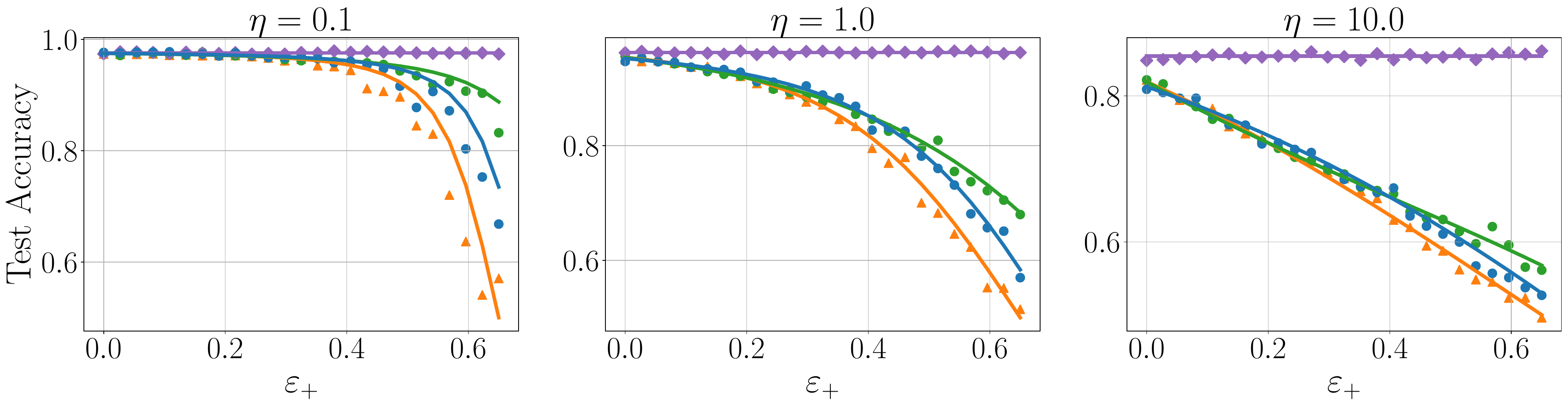}
    \includegraphics[width = \textwidth]{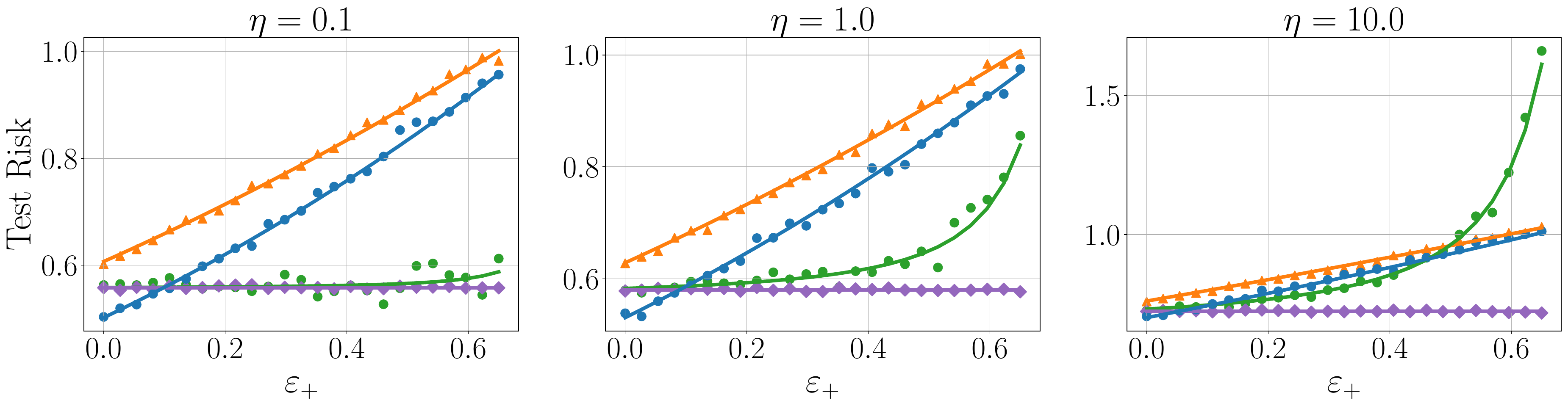}
    \includegraphics[width = .8\textwidth]{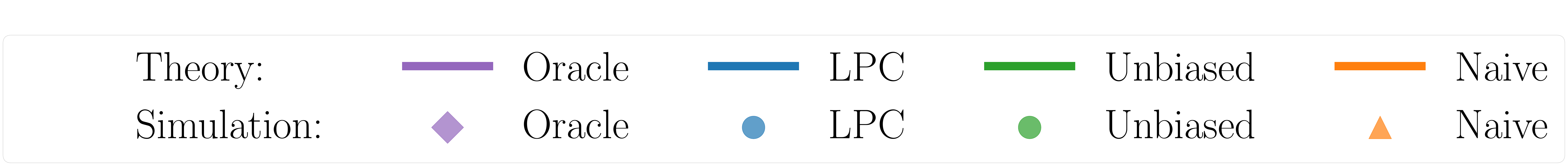}
    \caption{Test performance (accuracy and risk) of different LPC variants in terms of the positive noise rate $\varepsilon_+$. We considered $n = 100$, $\pi_1 = \frac13$, $\varepsilon_- = 0.2$, $\Vert \vmu \Vert = 2$, $\gamma = 10$, $\rho_+ = 0.2$ and $\rho_- = 0$ (for LPC in blue). The theoretical curves are obtained as per Proposition \ref{prop:test-accuracy}. We notice that the effect of label noise is more important in high-dimension, i.e., large values of $\eta$. }
    \label{fig:metric-vs-epsilon}
\end{figure}


\begin{proposition}[Asymptotic test accuracy \& risk of LPC]
\label{prop:test-accuracy}
The asymptotic test accuracy and risk of LPC $\vw_\rho$ in \eqref{w_imp}, under Assumption \ref{assum_growth_rate} and as $n_{\text{test}}\to \infty$, are respectively given by:
\begin{align*}
    \gA_{\text{test}} \asto 1 - \varphi\left( (\nu_\rho - m_\rho^2)^{-\frac12} m_\rho \right), \quad \gR_{\text{test}} \asto 1 - 2 m_\rho + \nu_\rho.
\end{align*}
where $m_\rho, \nu_\rho$ are defined in Theorem \ref{thm_main} and $\varphi(x) = \frac{1}{ \sqrt{2\pi} } \int_x^{+\infty} e^{-\frac{t^2}{2}} \mathrm{dt} $. 
\end{proposition}
Figure \ref{fig:metric-vs-epsilon} depicts both the empirical and theoretical test performance of LPC and its different variants, where we essentially notice a very accurate matching between simulation and the theoretical predictions as per Proposition \ref{prop:test-accuracy}, even for a finite sample size. See Figure \ref{fig:accuracy-gamma} in the Appendix for more plots varying other parameters.
In fact, even though we work under an asymptotic regime, our estimation of $\gA_{\text{test}}$ and $\gR_{\text{test}}$ by their asymptotic counterparts is consistent, as it can be shown that their fluctuations are of order $n^{-\frac12}$ under Assumption \ref{assum_growth_rate}, this is a consequence of the concentration results of the resolvent $\rmQ$ as shown in \citep{louart2018concentration}.

 Interestingly, when observing the asymptotic test accuracy in terms of $\rho_+$ and $\rho_-$ as depicted in Figure \ref{fig:comp-rho-epsilon}, we remarkably find that the accuracy is maximized for any fixed $\rho_-$ at some value $\rho_+^*(\rho_-)$, and the maximum accuracy is higher than the \textit{unbiased} accuracy in high-dimension. Moreover, since $\varphi(\cdot)$ is monotonous, such maximizer can be obtained analytically by maximizing the ratio $(\nu_\rho - m_\rho^2)^{-\frac12} m_\rho$ as derived in Appendix \ref{appendix_maximal_rho}, which yields the following closed-form expression:
\begin{align}\label{eq_optimal_rho}
    \boxed{\rho_+^*(\rho_-) = \frac{\pi_1^2 \varepsilon_-(\varepsilon_- - 1) + \pi_2^2 \varepsilon_+(1 - \varepsilon_+)}{\pi_1 \pi_2(1 - \varepsilon_+ - \varepsilon_-)} + \rho_-.}
\end{align}
Therefore, our \textit{optimized classifier} is defined by taking $\rho_-=0$ and $\rho_+ = \rho_+^*(0)$ in the expression of $\vw_\rho$ as per \eqref{w_imp}. We notably notice that $\rho_+^*$ depends solely on the noise probabilities $\varepsilon_\pm$ and the class proportions $\pi_1$ and $\pi_2$, especially, it does not involve the SNR $\Vert \vmu\Vert$, the regularization $\gamma$ and the dimension ratio $\eta$ which is quite unexpected. We also notice that the worst performance of LPC with parameters $\bar\rho_+, \bar\rho_-$ (again $\bar\rho_-$ can be fixed to $0$) corresponds to the one of a random guess and can be derived by solving $m_\rho=0$ which yields (for $\pi_1\neq \frac12$):
\begin{align}
    \bar \rho_+(\rho_-) =  \frac{1 - 2 \pi_1 \varepsilon_- - 2\pi_2 \varepsilon_+}{2 \pi_1 - 1} +  \rho_-.
\end{align}

\begin{figure}[t!]
    \centering
    \includegraphics[width = \textwidth]{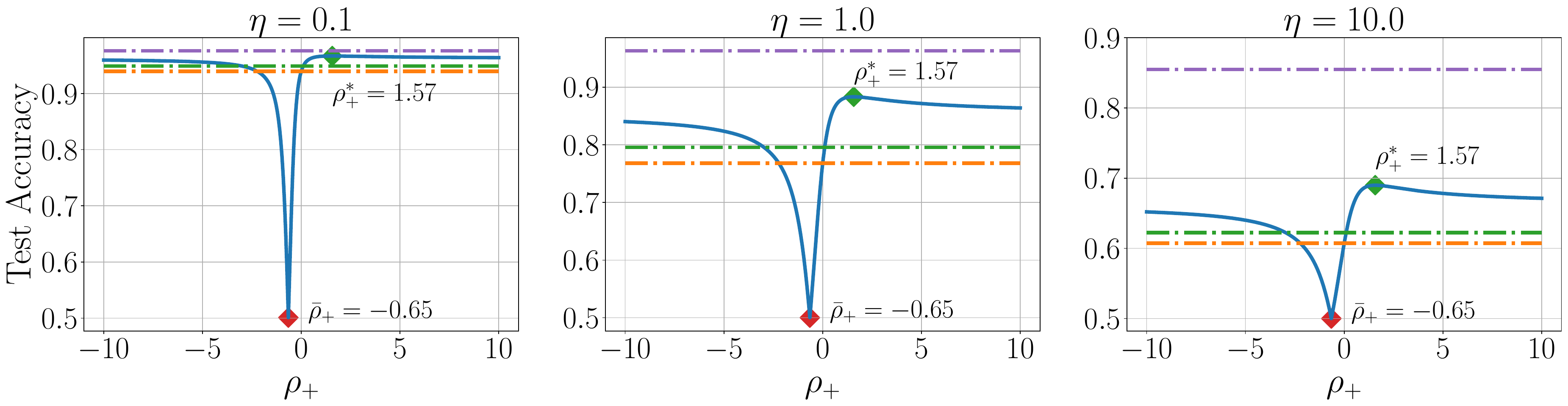}
    
    \includegraphics[width = 0.85\textwidth]{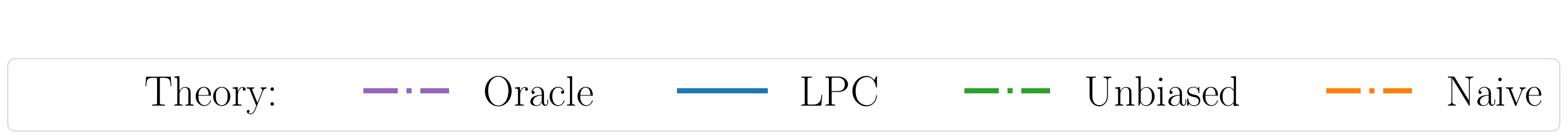}
    \caption{Test accuracy of LPC by fixing $\rho_-=0$ and varying $\rho_+$. We considered $n = 1000$, $\pi_1 = 0.3$, $\Vert \vmu \Vert = 2$, $\varepsilon_+ = 0.4$, $\varepsilon_- = 0.3$ and optimal $\gamma$. We notice that the test accuracy is maximized at $\rho_+^*$ yielding better accuracy compared with the \textit{unbiased} approach. Note that for small values of $\eta$, i.e., for low dimensions, the test accuracy becomes flat in terms of $\rho_+$ and in the limit $\eta\to 0$ the maximizer $\rho_+^*$ is not identifiable as discussed in Remark \ref{remark_infinie_samples}.}
    \label{fig:comp-rho-epsilon}
\end{figure}

\begin{remark}[On the relevance of the RMT analysis]\label{remark_infinie_samples}
    Our RMT analysis relies on the main assumption that both $p$ and $n$ are large and comparable as per Assumption \ref{assum_growth_rate}. This assumption is in fact fundamentally crucial for exhibiting the maximizer $\rho_+^*$ defined above. Indeed, supposing an infinite sample size setting where $p$ is fixed while taking only $n\to \infty$ or alternatively $h\to 1$, would result in $(\nu_\rho - m_\rho^2)^{-\frac12} m_\rho \to \Vert \vmu \Vert$. Therefore, the existence of $\rho_+^*$ is only tractable under the large dimensional setting, which motivates the importance of this assumption.
\end{remark}
\subsection{Estimation of label noise probabilities}\label{sec_noise_estimation}

Another important aspect of our \textit{optimized} classifier is the fact that it supposes the prior knowledge of the noise probabilities $\varepsilon_\pm$ which is also the case for the \textit{unbiased} classifier of \citep{natarajan2018cost}. In this section, based on our theoretical derivations, we propose a simple procedure for estimating $\varepsilon_\pm$ by supposing that the SNR $\Vert \vmu \Vert$ and the class proportions $\pi_1, \pi_2$ are known, in fact the latest can be consistently estimated with very few training samples as described in \citep{tiomoko2021deciphering}.

To estimate $\varepsilon_\pm$, we rely on the expression of the second moment $\nu_\rho = \nu_\rho(\varepsilon_+, \varepsilon_-)$ as per Theorem \ref{thm_main}, by viewing it as a function of $\varepsilon_\pm$. Specifically, we consider two different arbitrary couples $\rho_1 = (\rho_+^1, \rho_-^1)$ and $\rho_2 = (\rho_+^2, \rho_-^2)$ and solve the system:
\begin{equation}
    \begin{cases}
        \hat \nu_{\rho_1} = \nu_{\rho_1} (\varepsilon_+, \varepsilon_-),\\
        \hat \nu_{\rho_2} = \nu_{\rho_2} (\varepsilon_+, \varepsilon_-).
    \end{cases}
\end{equation}
where $\hat \nu_{\rho} = \frac1n \sum_{i=1}^n ( \vx_i^\top\vw_\rho^{-i} )^2 $ is the empirical estimate of $\nu_\rho$ and $\vw_\rho^{-i}$ corresponds to LPC trained on all examples except $\vx_i$, which discards the statistical dependencies. Figure \ref{fig:epsilon-estimation} depicts the estimated versus ground truth value of $\varepsilon_+$ and shows consistent estimation for different values of the SNR $\Vert \vmu \Vert$.

\section{Experiments with real data}\label{sec_experiments}

In this section, we present experiments with real data to validate our approach. We use the Amazon review dataset \citep{blitzer2007biographies} which includes several binary classification tasks corresponding to positive versus negative reviews of \texttt{books}, \texttt{dvd}, \texttt{electronics} and \texttt{kitchen}. We apply the standard scaler from \texttt{sklearn} \citep{pedregosa2011scikit} and estimate $\Vert \vmu \Vert$ with the normalized data. Figure \ref{fig:distribution-book} depicts the histogram of the decision function of different LPC variants (\textit{Naive}, \textit{Unbiased} and \textit{Optimized}) along with the theoretical distribution as predicted by Theorem \ref{thm_main}. We notably observe a reasonable match between the empirical histograms and the theoretical predictions which validates our results and assumptions even on real data. Note that, even though we considered a Gaussian mixture model, our results extend beyond this assumption as we discussed in Remark \ref{remark_distribution}. In fact, our results can be derived under the more general setting of concentrated random vectors \citep{louart2018concentration} which typically accounts GAN generated data \citep{seddik2020random}. 

From a practical standpoint, we highlight that we estimate the SNR $\Vert \vmu \Vert$ on the real data only for plotting the theoretical distributions in Figure \ref{fig:distribution-book}. In fact, our \textit{optimized} classifier does not require the knowledge of $\Vert \vmu \Vert$ since $\rho_+^*$ depends only on the class proportions $\pi_a$'s and the noise probabilities $\varepsilon_\pm$ as per \eqref{eq_optimal_rho}.
However, if the latest quantities are unknown, one can estimate them as we discussed in the previous section and therefore the knowledge of $\Vert \vmu \Vert$ is required, but can also be consistently estimated with few data samples as discussed earlier. 
Moreover, as theoretically anticipated, the \textit{optimized classifier} outperforms the \textit{naive} and \textit{unbiased} classifiers in terms of accuracy. Table \ref{accuracy-table} shows the performance in terms of classification accuracy of the different classifiers, on different datasets and varying the noise probabilities. We clearly observe that the \textit{optimized} approach yields spectacular performances which are almost close to the \textit{oracle} that assumes perfect knowledge of the true labels, even under a high noise regime. The code is provided in the supplementary material for reproducibility of our empirical results.

\begin{figure}[t!]
\centering
\includegraphics[width=\textwidth]{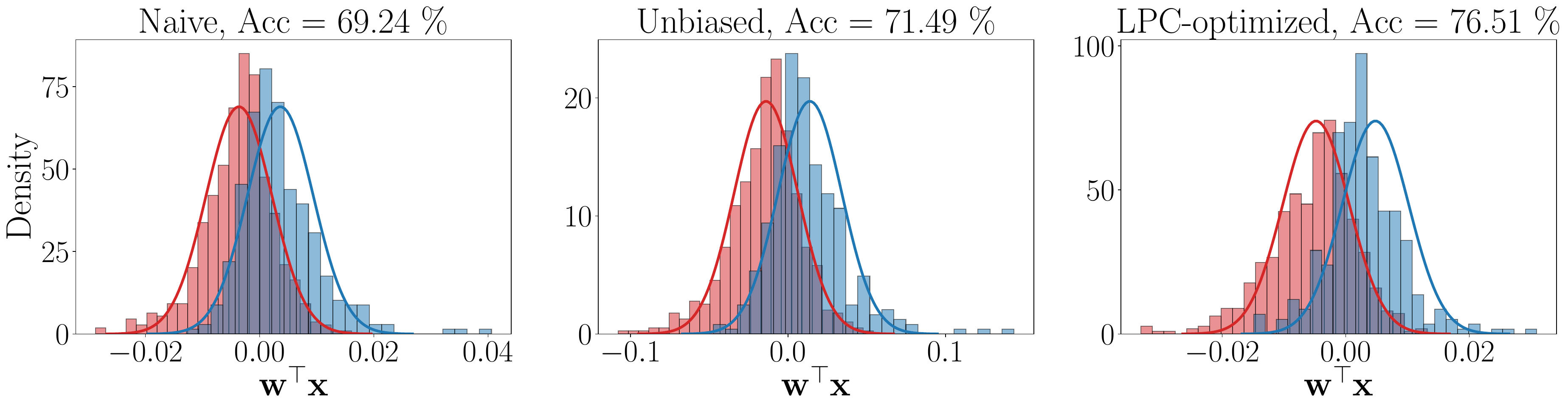}
\includegraphics[width=.8\textwidth]{img/legend_1.pdf}
\caption{Histogram of the decision function of different LPC variants on the \texttt{books} dataset \citep{blitzer2007biographies}, along with the theoretical distribution as predicted by Theorem \ref{thm_main}. We considered $n = 1600$, $p = 400$, $\pi_1 = 0.3$, $\varepsilon_+ = 0.4$, $\varepsilon_- = 0.3$ and optimal $\gamma$.}
\label{fig:distribution-book}
\end{figure}

\begin{table}[t!]
  \caption{Accuracy comparison over Amazon review datasets \citep{blitzer2007biographies} for $ n = 1600$, $ p = 400$, $\pi_1 = 0.3$, $\varepsilon_- = 0.4$ and optimal $\gamma$. As theoretically anticipated, our \textit{optimized} approach yields better classification accuracy even approaching \textit{oracle} trained with the true labels.}
  \label{accuracy-table}
  \vspace{.1cm}
  \centering
  \begin{tabular}{llllll}
    \toprule
    \textbf{$\varepsilon_+$} & \textbf{Sub-Dataset}     & \textbf{Naive} (\%)     & \textbf{Unbiased} (\%)      & \textbf{Optimized} (\%)  &
    \textbf{Oracle} (\%) \\
    \midrule
    0.3 & \texttt{Books} & 72.69 $\pm$ 0.11   & 71.66 $\pm$ 0.25 &  $\mathcolorbox{green!50}{76.36 \pm 0.21}$  & {\color{gray} 78.78 $\pm$ 0.07} \\
     & \texttt{Dvd} & 73.75 $\pm$ 0.42 & 72.24 $\pm$ 0.3 & $\mathcolorbox{green!50}{77.43 \pm 0.04}$ & {\color{gray} 80.57 $\pm$ 0.12} \\
     &\texttt{Electronics} & 78.22 $\pm$ 0.05 & 77.22 $\pm$ 0.09 & $\mathcolorbox{green!50}{81.57 \pm 0.12}$ & {\color{gray} 83.22 $\pm$ 0.09} \\
    & \texttt{Kitchen} &  79.64 $\pm$ 0.07 & 78.62 $\pm$ 0.05 & $\mathcolorbox{green!50}{82.17 \pm 0.06}$  & {\color{gray} 84.28 $\pm$ 0.1}\\ \midrule

    0.4 & \texttt{Books} & 66.84 $\pm$ 0.31 & 66.68 $\pm$ 0.22 &   $\mathcolorbox{green!50}{75.69 \pm 0.22}$ & {\color{gray} 78.78 $\pm$ 0.07}\\
     & \texttt{Dvd} & 67.2 $\pm$ 0.37 & 67.33 $\pm$ 0.34 & $\mathcolorbox{green!50}{76.86 \pm 0.16}$ & {\color{gray} 80.57 $\pm$ 0.12}  \\
     &\texttt{Electronics} & 72.13 $\pm$ 0.18 & 72.36 $\pm$ 0.06 & $\mathcolorbox{green!50}{81.04 \pm 0.08}$ & {\color{gray} 83.22 $\pm$ 0.09} \\
    & \texttt{Kitchen} & 73.46 $\pm$ 0.29 & 73.85 $\pm$ 0.23 & $\mathcolorbox{green!50}{81.65 \pm 0.17}$ & {\color{gray} 84.28 $\pm$ 0.1} \\
    \midrule
    
    0.5 & \texttt{Books} & 55.37 $\pm$ 0.25 & 59.5 $\pm$ 0.43 & $\mathcolorbox{green!50}{75.26 \pm 0.19}$ & {\color{gray} 78.78 $\pm$ 0.07} \\
     & \texttt{Dvd} & 55.32 $\pm$ 0.41 & 59.68 $\pm$ 0.57 & $\mathcolorbox{green!50}{76.42 \pm 0.13}$ & {\color{gray} 80.57 $\pm$ 0.12} \\
     &\texttt{Electronics} & 57.96 $\pm$ 0.11 & 63.21 $\pm$ 0.36 & $\mathcolorbox{green!50}{80.73 \pm 0.01 }$ & {\color{gray} 83.22 $\pm$ 0.09} \\
    & \texttt{Kitchen} & 58.15 $\pm$ 0.61 & 64.71 $\pm$ 0.7 & $\mathcolorbox{green!50}{81.32 \pm 0.11}$ & {\color{gray} 84.28 $\pm$ 0.1} \\
    
    \bottomrule
  \end{tabular}
\end{table}

\section{Conclusion \& future directions}\label{sec_conclusion}

This paper introduced new insights into learning with noisy labels in high dimensions. Relying on tools from random matrix theory, we provided an asymptotic characterization of the performance of the introduced classifier which accounts for label noise through scalar quantities. 
Based on this analysis, we identified that the low-dimensional intuitions to handle label noise do not extend to high-dimension and developed a new approach that is proven to be more efficient by design. 
We also showed through empirical evidence that our approach yields improved performance on real data.

In our current investigation, we restricted our analysis to the cases of squared loss and binary classification. Our results can be extended beyond these settings by accounting for a general bounded loss function $\ell(s, y)$ and multi-class classification problems. We provide in Appendix \ref{appendix_bce_loss} some experiments with synthetic and real data using the binary-cross-entropy loss function that show similar behavior to the squared loss (see Figures \ref{fig:bce-loss-synthetic} and \ref{fig:bce-loss-real}), namely, the existence of an optimum $\rho_\pm^*$ that outperforms the \textit{unbiased} approach in high dimensions. The extension of our study to this setting can be performed by leveraging the empirical risk minimization framework \citep{el2013robust, mai2019high} which allows the RMT analysis of general loss functions. Moreover, as we provided in Appendix \ref{appendix_multi_class}, our results extend to a $k$-class classification setting where we empirically show improved performance by optimizing a set of $2k$ scalar parameters (which play the same role as $\rho_\pm$ of the binary case). Such extension is straightforward in the case of squared loss given our current results and will be addressed in future work.

%

\bibliographystyle{bibstyle}
\bibliography{ref.bib}

\begin{thebibliography}{33}
\providecommand{\natexlab}[1]{#1}
\providecommand{\url}[1]{\texttt{#1}}
\expandafter\ifx\csname urlstyle\endcsname\relax
  \providecommand{\doi}[1]{doi: #1}\else
  \providecommand{\doi}{doi: \begingroup \urlstyle{rm}\Url}\fi

\bibitem[Bai \& Silverstein(2010)Bai and Silverstein]{bai2010spectral}
Zhidong Bai and Jack~W Silverstein.
\newblock \emph{Spectral analysis of large dimensional random matrices}, volume~20.
\newblock Springer, 2010.

\bibitem[Biggio et~al.(2011)Biggio, Nelson, and Laskov]{biggio2011support}
Battista Biggio, Blaine Nelson, and Pavel Laskov.
\newblock Support vector machines under adversarial label noise.
\newblock In \emph{Asian conference on machine learning}, pp.\  97--112. PMLR, 2011.

\bibitem[Blitzer et~al.(2007)Blitzer, Dredze, and Pereira]{blitzer2007biographies}
John Blitzer, Mark Dredze, and Fernando Pereira.
\newblock Biographies, bollywood, boom-boxes and blenders: Domain adaptation for sentiment classification.
\newblock In \emph{Proceedings of the 45th annual meeting of the association of computational linguistics}, pp.\  440--447, 2007.

\bibitem[Couillet \& Benaych-Georges(2016)Couillet and Benaych-Georges]{couillet2016kernel}
Romain Couillet and Florent Benaych-Georges.
\newblock Kernel spectral clustering of large dimensional data.
\newblock 2016.

\bibitem[Couillet \& Liao(2022)Couillet and Liao]{couillet2022random}
Romain Couillet and Zhenyu Liao.
\newblock \emph{Random matrix methods for machine learning}.
\newblock Cambridge University Press, 2022.

\bibitem[Crammer \& Lee(2010)Crammer and Lee]{crammer2010learning}
Koby Crammer and Daniel Lee.
\newblock Learning via gaussian herding.
\newblock \emph{Advances in neural information processing systems}, 23, 2010.

\bibitem[Crammer et~al.(2006)Crammer, Dekel, Keshet, Shalev-Shwartz, Singer, and Warmuth]{crammer2006online}
Koby Crammer, Ofer Dekel, Joseph Keshet, Shai Shalev-Shwartz, Yoram Singer, and Manfred~K Warmuth.
\newblock Online passive-aggressive algorithms.
\newblock \emph{Journal of Machine Learning Research}, 7\penalty0 (3), 2006.

\bibitem[Crammer et~al.(2009)Crammer, Kulesza, and Dredze]{crammer2009adaptive}
Koby Crammer, Alex Kulesza, and Mark Dredze.
\newblock Adaptive regularization of weight vectors.
\newblock \emph{Advances in neural information processing systems}, 22, 2009.

\bibitem[Dandi et~al.(2024)Dandi, Stephan, Krzakala, Loureiro, and Zdeborov{\'a}]{dandi2024universality}
Yatin Dandi, Ludovic Stephan, Florent Krzakala, Bruno Loureiro, and Lenka Zdeborov{\'a}.
\newblock Universality laws for gaussian mixtures in generalized linear models.
\newblock \emph{Advances in Neural Information Processing Systems}, 36, 2024.

\bibitem[Dredze et~al.(2008)Dredze, Crammer, and Pereira]{dredze2008confidence}
Mark Dredze, Koby Crammer, and Fernando Pereira.
\newblock Confidence-weighted linear classification.
\newblock In \emph{Proceedings of the 25th international conference on Machine learning}, pp.\  264--271, 2008.

\bibitem[El~Karoui et~al.(2013)El~Karoui, Bean, Bickel, Lim, and Yu]{el2013robust}
Noureddine El~Karoui, Derek Bean, Peter~J Bickel, Chinghway Lim, and Bin Yu.
\newblock On robust regression with high-dimensional predictors.
\newblock \emph{Proceedings of the National Academy of Sciences}, 110\penalty0 (36):\penalty0 14557--14562, 2013.

\bibitem[Freund(2009)]{freund2009more}
Yoav Freund.
\newblock A more robust boosting algorithm.
\newblock \emph{arXiv preprint arXiv:0905.2138}, 2009.

\bibitem[Graepel \& Herbrich(2000)Graepel and Herbrich]{graepel2000kernel}
Thore Graepel and Ralf Herbrich.
\newblock The kernel gibbs sampler.
\newblock \emph{Advances in Neural Information Processing Systems}, 13, 2000.

\bibitem[Hachem et~al.(2007)Hachem, Loubaton, and Najim]{hachem2007deterministic}
Walid Hachem, Philippe Loubaton, and Jamal Najim.
\newblock Deterministic equivalents for certain functionals of large random matrices.
\newblock 2007.

\bibitem[Jiang(2001)]{jiang2001some}
Wenxin Jiang.
\newblock Some theoretical aspects of boosting in the presence of noisy data.
\newblock In \emph{Proceedings of the Eighteenth International Conference on Machine Learning}. Citeseer, 2001.

\bibitem[Karimi et~al.(2020)Karimi, Dou, Warfield, and Gholipour]{karimi2020deep}
Davood Karimi, Haoran Dou, Simon~K Warfield, and Ali Gholipour.
\newblock Deep learning with noisy labels: Exploring techniques and remedies in medical image analysis.
\newblock \emph{Medical image analysis}, 65:\penalty0 101759, 2020.

\bibitem[Lawrence \& Sch{\"o}lkopf(2001)Lawrence and Sch{\"o}lkopf]{lawrence2001estimating}
Neil Lawrence and Bernhard Sch{\"o}lkopf.
\newblock Estimating a kernel fisher discriminant in the presence of label noise.
\newblock In \emph{18th international conference on machine learning (ICML 2001)}, pp.\  306--306. Morgan Kaufmann, 2001.

\bibitem[Li et~al.(2020)Li, Socher, and Hoi]{li2020dividemix}
Junnan Li, Richard Socher, and Steven~CH Hoi.
\newblock Dividemix: Learning with noisy labels as semi-supervised learning.
\newblock \emph{arXiv preprint arXiv:2002.07394}, 2020.

\bibitem[Li et~al.(2007)Li, Wessels, de~Ridder, and Reinders]{li2007classification}
Yunlei Li, Lodewyk~FA Wessels, Dick de~Ridder, and Marcel~JT Reinders.
\newblock Classification in the presence of class noise using a probabilistic kernel fisher method.
\newblock \emph{Pattern Recognition}, 40\penalty0 (12):\penalty0 3349--3357, 2007.

\bibitem[Louart \& Couillet(2018)Louart and Couillet]{louart2018concentration}
Cosme Louart and Romain Couillet.
\newblock Concentration of measure and large random matrices with an application to sample covariance matrices.
\newblock \emph{arXiv preprint arXiv:1805.08295}, 2018.

\bibitem[Ma et~al.(2018)Ma, Wang, Houle, Zhou, Erfani, Xia, Wijewickrema, and Bailey]{ma2018dimensionality}
Xingjun Ma, Yisen Wang, Michael~E Houle, Shuo Zhou, Sarah Erfani, Shutao Xia, Sudanthi Wijewickrema, and James Bailey.
\newblock Dimensionality-driven learning with noisy labels.
\newblock In \emph{International Conference on Machine Learning}, pp.\  3355--3364. PMLR, 2018.

\bibitem[Mai \& Liao(2019)Mai and Liao]{mai2019high}
Xiaoyi Mai and Zhenyu Liao.
\newblock High dimensional classification via regularized and unregularized empirical risk minimization: Precise error and optimal loss.
\newblock \emph{arXiv preprint arXiv:1905.13742}, 2019.

\bibitem[Manwani \& Sastry(2013)Manwani and Sastry]{manwani2013noise}
Naresh Manwani and PS~Sastry.
\newblock Noise tolerance under risk minimization.
\newblock \emph{IEEE transactions on cybernetics}, 43\penalty0 (3):\penalty0 1146--1151, 2013.

\bibitem[Mei \& Montanari(2022)Mei and Montanari]{mei2022generalization}
Song Mei and Andrea Montanari.
\newblock The generalization error of random features regression: Precise asymptotics and the double descent curve.
\newblock \emph{Communications on Pure and Applied Mathematics}, 75\penalty0 (4):\penalty0 667--766, 2022.

\bibitem[Nakkiran et~al.(2021)Nakkiran, Kaplun, Bansal, Yang, Barak, and Sutskever]{nakkiran2021deep}
Preetum Nakkiran, Gal Kaplun, Yamini Bansal, Tristan Yang, Boaz Barak, and Ilya Sutskever.
\newblock Deep double descent: Where bigger models and more data hurt.
\newblock \emph{Journal of Statistical Mechanics: Theory and Experiment}, 2021\penalty0 (12):\penalty0 124003, 2021.

\bibitem[Natarajan et~al.(2013)Natarajan, Dhillon, Ravikumar, and Tewari]{natarajan2013learning}
Nagarajan Natarajan, Inderjit~S Dhillon, Pradeep~K Ravikumar, and Ambuj Tewari.
\newblock Learning with noisy labels.
\newblock \emph{Advances in neural information processing systems}, 26, 2013.

\bibitem[Natarajan et~al.(2018)Natarajan, Dhillon, Ravikumar, and Tewari]{natarajan2018cost}
Nagarajan Natarajan, Inderjit~S Dhillon, Pradeep Ravikumar, and Ambuj Tewari.
\newblock Cost-sensitive learning with noisy labels.
\newblock \emph{Journal of Machine Learning Research}, 18\penalty0 (155):\penalty0 1--33, 2018.

\bibitem[Pedregosa et~al.(2011)Pedregosa, Varoquaux, Gramfort, Michel, Thirion, Grisel, Blondel, Prettenhofer, Weiss, Dubourg, et~al.]{pedregosa2011scikit}
Fabian Pedregosa, Ga{\"e}l Varoquaux, Alexandre Gramfort, Vincent Michel, Bertrand Thirion, Olivier Grisel, Mathieu Blondel, Peter Prettenhofer, Ron Weiss, Vincent Dubourg, et~al.
\newblock Scikit-learn: Machine learning in python.
\newblock \emph{the Journal of machine Learning research}, 12:\penalty0 2825--2830, 2011.

\bibitem[Scott et~al.(2013)Scott, Blanchard, and Handy]{scott2013classification}
Clayton Scott, Gilles Blanchard, and Gregory Handy.
\newblock Classification with asymmetric label noise: Consistency and maximal denoising.
\newblock In \emph{Conference on learning theory}, pp.\  489--511. PMLR, 2013.

\bibitem[Seddik et~al.(2020)Seddik, Louart, Tamaazousti, and Couillet]{seddik2020random}
Mohamed El~Amine Seddik, Cosme Louart, Mohamed Tamaazousti, and Romain Couillet.
\newblock Random matrix theory proves that deep learning representations of gan-data behave as gaussian mixtures.
\newblock In \emph{International Conference on Machine Learning}, pp.\  8573--8582. PMLR, 2020.

\bibitem[Song et~al.(2022)Song, Kim, Park, Shin, and Lee]{song2022learning}
Hwanjun Song, Minseok Kim, Dongmin Park, Yooju Shin, and Jae-Gil Lee.
\newblock Learning from noisy labels with deep neural networks: A survey.
\newblock \emph{IEEE transactions on neural networks and learning systems}, 2022.

\bibitem[Tanaka et~al.(2018)Tanaka, Ikami, Yamasaki, and Aizawa]{tanaka2018joint}
Daiki Tanaka, Daiki Ikami, Toshihiko Yamasaki, and Kiyoharu Aizawa.
\newblock Joint optimization framework for learning with noisy labels.
\newblock In \emph{Proceedings of the IEEE conference on computer vision and pattern recognition}, pp.\  5552--5560, 2018.

\bibitem[Tiomoko et~al.(2021)Tiomoko, Tiomoko, and Couillet]{tiomoko2021deciphering}
Malik Tiomoko, Hafiz Tiomoko, and Romain Couillet.
\newblock Deciphering and optimizing multi-task learning: a random matrix approach.
\newblock In \emph{ICLR 2021-9th International Conference on Learning Representations}, 2021.

\end{thebibliography}

\appendix
\newpage

\section*{Appendix}
This appendix is organized as follows: Section \ref{appendix_lemmas} lists some useful lemmas that will be at the core of our analysis. 
In Section \ref{appendix_rmt_analysis}, we provide a more general result of Theorem \ref{thm_main} as discussed in Remark \ref{remark_distribution} along with the main proof derivations using RMT.
Section \ref{appendix_plots} provides additional plots to support our theoretical results.
Section \ref{appendix_maximal_rho} provides derivations for finding the optimal parameter $\rho^*_+$ which defines our optimized classifier.
In Section \ref{appendix_bce_loss} we provide some experiments with synthetic and real data to support the extension of our analysis to arbitrary loss functions instead of the squared loss as supposed in the main paper.
Finally, Section \ref{appendix_multi_class} presents experiments showing that our analysis can be further extended to multi-class classification which is left for a future investigation.

\section{Useful lemmas}\label{appendix_lemmas}

The following lemmas will be useful in the calculus introduced in this section.

\begin{lemma}[Resolvent identity]\label{resolvent-identity}
    For invertible matrices $\rmA$ and $\rmB $, we have:
    \begin{equation*}
        \rmA^{-1} - \rmB^{-1} = \rmA^{-1}(\rmB - \rmA)\rmB^{-1}.
    \end{equation*}
\end{lemma}

\begin{lemma}[Sherman-Morisson]
\label{Sherman-Morisson}
For $\rmA \in \sR^{p \times p}$ invertible and $\vu, \vv \in \sR^p$, $\rmA  + \vu \vv^\top $ is invertible if and only if:  $1 + \vv^\top \rmA^{-1} \vu \neq 0$, and:
\begin{equation*}
    (\rmA + \vu \vv^\top)^{-1} = \rmA^{-1} - \frac{\rmA^{-1} \vu \vv^\top \rmA^{-1}}{1 + \vv^\top \rmA^{-1} \vu}.
\end{equation*}
Besides,
\begin{equation*}
    (\rmA + \vu \vv^\top)^{-1} \vu = \frac{\rmA^{-1} \vu}{1 + \vv^\top \rmA^{-1} \vu}.
\end{equation*}
\end{lemma}

\begin{lemma}[Relevant Identities]
\label{mu-identities}
    Let $\bar \rmQ \in \sR^{p \times p} $ be the deterministic matrix defined in lemma \ref{DEQ}. If $\rmC_a = \rmI_p$, then we have:
    \begin{align*}
        \vmu^\top \bar \rmQ \vmu = \frac{(1 + \delta) \Vert \vmu \Vert^2}{\Vert \vmu \Vert^2 + 1 + \gamma (1 + \delta) },\quad \vmu^\top \bar \rmQ^2 \vmu = \left (\frac{(1 + \delta) \Vert \vmu \Vert}{\Vert \vmu \Vert^2 + 1 + \gamma (1 + \delta)} \right)^2.
    \end{align*}
\end{lemma}

\begin{proof}
We have that:
\begin{align*}
    \bar \rmQ &= \left ( \frac{\vmu \vmu^\top}{1 + \delta} + \left ( \gamma + \frac{1}{1 + \delta} \right) \rmI_p \right )^{-1} \\
    &= (1 + \delta) \left ( \vmu \vmu^\top + (1 + \gamma (1 + \delta) \rmI_p )\right)^{-1} \\
    &= \frac{1 + \delta}{1 + \gamma (1 + \delta)} \left ( \frac{\vmu \vmu^\top}{1 + \gamma (1 + \delta)} + \rmI_p \right)^{-1} \\
    &= \frac{1 + \delta}{1 + \gamma (1 + \delta)} \left ( \rmI_p - \frac{\vmu \vmu^\top}{\Vert \vmu \Vert^2 + 1 + \gamma (1 + \delta)} \right) & (\text{lemma \ref{Sherman-Morisson}})
\end{align*}
where the last equality is obtained using Sherman-Morisson's identity (lemma \ref{Sherman-Morisson}).
Hence,
\begin{align*}
    (\bar \rmQ)^2 &= \frac{(1 + \delta)^2}{(1 + \gamma (1 + \delta))^2} \left( \rmI_p + \frac{(\vmu \vmu^\top)^2}{(\Vert \vmu \Vert^2 + 1 + \gamma (1 + \delta))^2} - \frac{2 \vmu \vmu^\top}{\Vert \vmu \Vert^2 + 1 + \gamma (1 + \delta)}\right)
\end{align*}
\paragraph{First identity:}
    \begin{align*}
        \vmu^\top \bar \rmQ \vmu &= \frac{(1 + \delta)}{(1 + \gamma (1 + \delta))} \left ( \Vert \vmu \Vert^2 - \frac{\Vert \vmu \Vert^4}{\Vert \vmu \Vert^2 + 1 + \gamma (1 + \delta)} \right) \\
        &= \frac{(1 + \delta) \Vert \vmu \Vert^2}{\Vert \vmu \Vert^2 + 1 + \gamma (1 + \delta)}
    \end{align*}

\paragraph{Second identity:}
\begin{align*}
    \vmu^\top \bar \rmQ^2 \vmu &= \frac{(1 + \delta)^2}{(1 + \gamma (1 + \delta))^2} \left ( \Vert \vmu \Vert^2 + \frac{\Vert \vmu \Vert^6}{(\Vert \vmu \Vert^2 + 1 + \gamma (1 + \delta))^2} - \frac{2 \Vert \vmu \Vert^4}{\Vert \vmu \Vert^2 + 1 + \gamma (1 + \delta)} \right) \\
    &= \frac{(1 + \delta)^2}{(1 + \gamma (1 + \delta))^2} \left ( \Vert \vmu \Vert - \frac{\Vert \vmu \Vert^3}{\Vert \vmu \Vert^2 + 1 + \gamma (1 + \delta)} \right)^2 \\
    &= \left (\frac{(1 + \delta) \Vert \vmu \Vert}{\Vert \vmu \Vert^2 + 1 + \gamma (1 + \delta)} \right)^2
\end{align*}
\end{proof}

\begin{lemma}[Deterministic equivalent of $\rmQ \rmA \rmQ$]
\label{de QAQ}
    For any positive semi-definite matrix $\rmA$, we have:
    \begin{equation*}
        \rmQ \rmA \rmQ \leftrightarrow \bar \rmQ \rmA \bar \rmQ + \frac{\pi_1}{n(1 + \delta_1)^2} \Tr(\Sigma_1 \bar \rmQ \rmA \bar \rmQ) \E[\rmQ \Sigma_1 \rmQ] + \frac{\pi_2}{n(1 + \delta_2)^2} \Tr(\Sigma_2 \bar \rmQ \rmA \bar \rmQ) \E[\rmQ \Sigma_2 \rmQ],
    \end{equation*}
    where $\Sigma_a = \vmu \vmu^\top + \rmC_a$. 
    In particular, if $\rmC = \rmI_p$, i.e $\Sigma = \vmu \vmu^\top + \rmI_p$ then:
    \begin{equation*}
        \rmQ \rmA \rmQ \leftrightarrow \bar \rmQ \rmA \bar \rmQ + \frac{1}{n} \frac{\Tr(\Sigma \bar \rmQ \rmA \bar \rmQ)}{(1 + \delta)^2} \E[\rmQ \Sigma \rmQ].
    \end{equation*}
\end{lemma}
\begin{proof}
    Let $\bar \rmQ$ be a d.e. of $\rmQ$. We have that:
    \begin{align*}
        \E[\rmQ \rmA \rmQ] &= \E[\bar \rmQ \rmA \rmQ] + \E[(\rmQ - \bar \rmQ) \rmA \rmQ]  \\
        &= \bar \rmQ (\E[\rmA \rmQ] + \rmA \E[\rmQ - \bar \rmQ]) + \E[(\rmQ - \bar \rmQ) \rmA \rmQ] \\
        &= \bar \rmQ \rmA \bar \rmQ + \E[(\rmQ - \bar \rmQ) \rmA \rmQ] 
    \end{align*}
    Using lemma \ref{resolvent-identity}, we have that:
    \begin{align*}
        \rmQ - \bar \rmQ &= \rmQ (\bar \rmQ^{-1} - \rmQ^{-1}) \bar \rmQ \\
        &= \rmQ \left (\pi_1 \frac{\Sigma_1}{1 + \delta_1} + \pi_2 \frac{\Sigma_2}{1 + \delta_2} - \frac{1}{n} \rmX \rmX^\top \right ) \bar \rmQ \\
        &= \rmQ(\rmS - \frac{1}{n} \rmX \rmX^\top) \bar \rmQ
    \end{align*}
    Thus:
    \begin{align*}
        \E[\rmQ \rmA \rmQ] &= \bar \rmQ \rmA \bar \rmQ + \E[\rmQ (\rmS - \frac{1}{n} \rmX \rmX^\top) \bar \rmQ \rmA \rmQ] \\
        &= \bar \rmQ \rmA \bar \rmQ + \E[\rmQ \rmS \bar \rmQ \rmA \rmQ] - \frac{1}{n} \sum_{i = 1}^n \E[\rmQ \vx_i \vx_i^\top \bar \rmQ \rmA \rmQ] \\
    \end{align*}
    We have that:
    \begin{align*}
        \E[\rmQ \vx_i \vx_i^\top \bar \rmQ \rmA \rmQ] &= \frac{1}{1 + \delta} \E[\rmQ_{-i} \vx_i \vx_i^\top \bar \rmQ \rmA \rmQ] \\
        &= \frac{1}{1 + \delta_i} \left( \E[\rmQ_{-i} \vx_i \vx_i^\top \bar \rmQ \rmA \rmQ_{-i}] - \E[\rmQ_{-i} \vx_i \vx_i^\top \bar \rmQ \rmA \frac{\rmQ_{-i} \vx_i \vx_i^\top \rmQ_{-i}}{n(1 + \delta_i)}] \right) \\
        &= \frac{1}{1 + \delta_i} \left( \E[\rmQ_{-i} \Sigma_i \bar \rmQ \rmA \rmQ_{-i}] - \E[\rmQ_{-i} \vx_i \vx_i^\top \bar \rmQ \rmA \frac{\rmQ_{-i} \vx_i \vx_i^\top \rmQ_{-i}}{n(1 + \delta_i)}] \right) \\
        &= \frac{1}{1 + \delta_i} \left( \E[\rmQ \Sigma_i \bar \rmQ \rmA \rmQ] - \E[\rmQ_{-i} \vx_i \vx_i^\top \bar \rmQ \rmA \frac{\rmQ_{-i} \vx_i \vx_i^\top \rmQ_{-i}}{n(1 + \delta_i)}] \right) \\
    \end{align*}
    Hence, by replacing in the previous identity, we get:
    \begin{align*}
        \E[\rmQ \rmA \rmQ] &= \bar \rmQ \rmA \bar \rmQ + \frac{1}{n} \sum_{i=1}^n \frac{1}{(1 + \delta_i)^2} \E[\rmQ_{-i} \vx_i \frac{1}{n} \vx_i^\top \bar \rmQ \rmA \rmQ_{-i} \vx_i \vx_i^\top \rmQ_{-i}] \\
        &= \bar \rmQ \rmA \bar \rmQ + \frac{1}{n^2} \sum_{i=1}^n \frac{1}{(1 + \delta_i)^2}\Tr(\Sigma_i \bar \rmQ \rmA \bar \rmQ) \E[\rmQ_{-i} \vx_i \vx_i^\top \rmQ_{-i}] \\
        &= \bar \rmQ \rmA \bar \rmQ + \frac{1}{n^2} \sum_{i=1}^n \frac{1}{(1 + \delta_i)^2} \Tr(\Sigma_i \bar \rmQ \rmA \bar \rmQ) \E[\rmQ \Sigma_i \rmQ] \\
        &= \bar \rmQ \rmA \bar \rmQ + \frac{\pi_1}{n(1 + \delta_1)^2} \Tr(\Sigma_1 \bar \rmQ \rmA \bar \rmQ) \E[\rmQ \Sigma_1 \rmQ] + \frac{\pi_2}{n(1 + \delta_2)^2} \Tr(\Sigma_2 \bar \rmQ \rmA \bar \rmQ) \E[\rmQ \Sigma_2 \rmQ] \\
        &= \bar \rmQ \rmA \bar \rmQ + \sum_b \frac{\pi_b}{n(1 + \delta_b)^2} \Tr(\Sigma_b \bar \rmQ \rmA \bar \rmQ) \E[\rmQ \Sigma_b \rmQ]
    \end{align*}
    Hence, we conclude that:
    \begin{equation*}
        \rmQ \rmA \rmQ \leftrightarrow \bar \rmQ \rmA \bar \rmQ + \frac{\pi_1}{n(1 + \delta_1)^2} \Tr(\Sigma_1 \bar \rmQ \rmA \bar \rmQ) \E[\rmQ \Sigma_1 \rmQ] + \frac{\pi_2}{n(1 + \delta_2)^2} \Tr(\Sigma_2 \bar \rmQ \rmA \bar \rmQ) \E[\rmQ \Sigma_2 \rmQ]
    \end{equation*}
\end{proof}


\section{RMT Analysis of the Label-Perturbed Classifier}\label{appendix_rmt_analysis}
\paragraph{Notation:}
For $a \in \{1, 2 \}$, we denote by $\sI_a = \{ i \mid \vx_i \in \gC_a \}$, i.e, the set of indices of vectors belonging to class $\gC_a$. Furthermore, we denote $\Sigma_a = \E \left[ \vx \vx^\top \right]$ for $\vx\in \gC_a$. 

\begin{assumption}[Generalized growth rates]
    \label{assum_growth_rate_general}
    Suppose that as $p,n\to \infty$:
    \begin{center}
        1) $\frac{p}{n} \to \eta \in (0, \infty)$,\hspace{.7cm}
        2) $\frac{n_a}{n} \to \pi_a\in (0, 1)$, \hspace{.7cm}
        3) $\Vert \vmu \Vert = \gO(1)$, \hspace{.7cm}
        4) $\Vert \Sigma_a \Vert = \gO(1) $,
    \end{center}
    $\Vert \Sigma_a \Vert $ is the spectral norm of the matrix $\Sigma_a$.
\end{assumption}

We consider the LPC with regularization parameter $\gamma$ given by:
\begin{align}
    \vw_{\rho} = \frac1n \rmQ(\gamma) \rmX \rmD_{\rho} \tilde \vy, \quad \rmQ(z) = \left(\frac1n \rmX \rm \rmX^\top + z \rmI_p\right)^{-1},
\end{align}
where $\rmX = \left[ \vx_1, \ldots, \vx_n \right] \in \sR^{p\times n}$ and $\tilde \vy = (\tilde y_1, \ldots, \tilde y_n)^\top \in \sR^n$. \\

\begin{theorem}[Gaussianity of LPC generalized]\label{thm_general_covariance}
    Let $\vw_{\rho}$ be the LPC as defined in \eqref{w_imp} and suppose that Assumption \ref{assum_growth_rate_general} holds. The decision function $\vw_\rho^\top \vx$, on some test sample $\vx\in \gC_a$ independent from $\rmX$, satisfies:
    \begin{align*}
    \vw_{\rho}^\top \vx \,\, \toind  \,\, \gN\left( (-1)^a m_\rho,\,  \nu_\rho - m_\rho^2 \right),
\end{align*}
where:
\begin{align*}
    &m_{\rho} = \left( \pi_1 \frac{(\lambda_- - 2 \beta \varepsilon_-)}{1 + \delta_1} + \pi_2 \frac{(\lambda_+ - 2 \beta \varepsilon_+)}{1 + \delta_2} \right) \vmu^\top \bar \rmQ \vmu, \\
    & \nu_\rho = \left( \frac{\pi_1(\lambda_- - 2 \beta \varepsilon_-)}{1 + \delta_1} + \frac{\pi_2(\lambda_+ - 2 \beta \varepsilon_+)}{1 + \delta_2} \right)^2 \vmu^\top \E[\rmQ \Sigma_a \rmQ] \vmu \\
    &- \frac{T_1}{1 + \delta_1} \left( \left (\frac{\pi_1(\lambda_- - 2 \beta \varepsilon_-)}{1 + \delta_1} \right)^2 \vmu^\top \bar \rmQ \vmu + \frac{\pi_1 \pi_2(\lambda_+ - 2 \beta \varepsilon_+)(\lambda_- - 2 \beta \varepsilon_-)}{(1 + \delta_1)(1 + \delta_2)} \vmu^\top \bar \rmQ \vmu \right) \\
    & + \frac{\pi_1 (4 \beta^2 \varepsilon_- (\rho_+ - \rho_-) + \lambda_-^2)}{(1 + \delta_1)^2} T_1 + \frac{\pi_2 (4 \beta^2 \varepsilon_+ (\rho_- - \rho_+) + \lambda_+^2)}{(1 + \delta_2)^2} T_2 \\
    & - \frac{T_2}{1 + \delta_2} \left( \left (\frac{\pi_2(\lambda_+ - 2 \beta \varepsilon_+)}{1 + \delta_2} \right)^2 \vmu^\top \bar \rmQ \vmu + \frac{\pi_1 \pi_2(\lambda_+ - 2 \beta \varepsilon_+)(\lambda_- - 2 \beta \varepsilon_-)}{(1 + \delta_1)(1 + \delta_2)} \vmu^\top \bar \rmQ \vmu \right),
\end{align*}
where $T_b = \frac1n \Tr(\Sigma_b \E[\rmQ \Sigma_a \rmQ])$ for $b\in [2]$ and $\E[\rmQ \Sigma_a \rmQ]$ is computed with Lemma \ref{de QAQ}.
\end{theorem}

Let $g_\rho(\vx) = \vw_\rho^\top \vx$, to prove Theorem \ref{thm_general_covariance}, we need to compute the expectation and the variance of $g_\rho(\vx)$ which are developed below.

\subsection{Test Expectation} 
Denote by $ \tilde \lambda_i = \frac{1 - \rho_{-\tilde y_i} + \rho_{\tilde y_i}}{ 1 - \rho_+ - \rho_-} $, then $\vw_{\rho} = \frac{1}{n} \sum_{i = 1}^n \rmQ(\gamma) \tilde \lambda_i  \tilde y_i \vx_i $. \\
We have:
\begin{align*}
    \E \left[ g_{\rho}(\vx) \right]&= \E \left[ \vw_{\rho}^\top \vx \right] = \frac1n \sum_{i=1}^n \E \left[ \tilde \lambda_i \tilde y_i \vx_i^\top \rmQ \vx \right] = \frac1n \sum_{i\in \sI_1} \E \left[ \tilde \lambda_i \tilde y_i \vx_i^\top \rmQ \vx \right] + \frac1n \sum_{i\in \sI_2} \E \left[ \tilde \lambda_i \tilde y_i \vx_i^\top \rmQ \vx \right]\\
    &=\frac1n \sum_{i\in \sI_1} \E \left[ \tilde \lambda_i \tilde y_i \vx_i^\top \rmQ \vmu_a \mid \vx_i\in \gC_1 \right] + \frac1n \sum_{i\in \sI_2} \E \left[ \tilde \lambda_i \tilde y_i \vx_i^\top \rmQ \vmu_a \mid \vx_i\in \gC_2 \right]
\end{align*}

Recall that: 
\begin{align}
     \lambda_+ = \frac{1 - \rho_- + \rho_+}{1 - \rho_+ - \rho_-}, \quad \lambda_- = \frac{1 - \rho_+ + \rho_-}{1 - \rho_+ - \rho_-},  \quad \beta = \frac{\lambda_- + \lambda_+}{2}
\end{align}
Then:
\begin{align*}
    \E \left[ \tilde \lambda_i \tilde y_i \vx_i^\top \rmQ \vmu_a \mid \vx_i \in \gC_1 \right] &= \lambda_+ \varepsilon_{-} \E \left[ \vx_i^\top \rmQ \vmu_a \mid y_i = -1 \right] - \lambda_- (1 - \varepsilon_-) \E \left[ \vx_i^\top \rmQ \vmu_a \mid y_i = -1 \right]\\
    &=  ( (\lambda_+ + \lambda_- )\varepsilon_- - \lambda_-) \E \left[ \vx_i^\top \rmQ \vmu_a \mid y_i = -1 \right] \\
    &= ( 2\beta \varepsilon_- - \lambda_-) \E \left[ \vx_i^\top \rmQ \vmu_a \mid y_i = -1 \right] \\
    &= \frac{( 2\beta \varepsilon_- - \lambda_-)}{1 + \delta_1} \vmu_1 \bar \rmQ \vmu_a
\end{align*}

Similarly, we have:
\begin{align*}
    \E \left[ \tilde \lambda_i \tilde y_i \vx_i^\top \rmQ \vmu_a \mid \vx_i\in \gC_2 \right] &= \lambda_+ (1 - \varepsilon_{+}) \E \left[ \vx_i^\top \rmQ \vmu_a \mid \vx_i\in \gC_2 \right] - \lambda_- \varepsilon_+ \E \left[  \vx_i^\top \rmQ \vmu_a \mid \vx_i\in \gC_2 \right]\\
    &=  (\lambda_+ - 2\beta \varepsilon_+) \E \left[ \vx_i^\top \rmQ \vmu_a \mid \vx_i\in \gC_2 \right] \\
    & = \frac{(\lambda_+ - 2\beta \varepsilon_+)}{1 + \delta_2} \vmu_2 \bar \rmQ \vmu_a 
\end{align*}

Therefore,
\begin{align*}
    \E \left[ g_{\rho}(\vx) \mid \vx \in \gC_a \right] &= \pi_1 \frac{(2 \beta \varepsilon_- - \lambda_-)}{1 + \delta_1} \vmu_1^\top \bar \rmQ \vmu_a + \pi_2 \frac{(\lambda_+ - 2 \beta \varepsilon_+)}{1 + \delta_2} \vmu_2^\top \bar \rmQ \vmu_a\\
    &= (-1)^a \left( \pi_1 \frac{(\lambda_- - 2 \beta \varepsilon_-)}{1 + \delta_1} + \pi_2 \frac{(\lambda_+ - 2 \beta \varepsilon_+)}{1 + \delta_2} \right) \vmu^\top \bar \rmQ \vmu 
\end{align*}

\subsection{Test Variance}
To compute the variance of $g_{\rho}(\vx)$, it only remains to compute the term: $\E[g_{\rho}(\vx)^2]$.
\begin{align*}
    \E[ g_{\rho}(\vx)^2] &= \frac{1}{n^2} \sum_{i, j = 1}^n \E[\tilde \lambda_i \tilde \lambda_j \tilde y_i \tilde y_j \vx_i^\top \rmQ \vx \vx_j^\top \rmQ \vx] \\
    &= \frac{1}{n^2} \sum_{i \in \sI_1} \sum_{j \in \sI_1} \E[\tilde \lambda_i \tilde \lambda_j \tilde y_i \tilde y_j \vx_i^\top \rmQ \vx \vx_j^\top \rmQ \vx \mid x_i \in \gC_1, x_j \in \gC_1 ] \\
    &+ \frac{2}{n^2} \sum_{i \in \sI_1} \sum_{j \in \sI_2} \E[ \tilde \lambda_i \tilde \lambda_j \tilde y_i \tilde y_j \vx_i^\top \rmQ \vx \vx_j^\top \rmQ \vx \mid x_i \in \gC_1, x_j \in \gC_2 ] \\
    &+ \frac{1}{n^2} \sum_{i \in \sI_2} \sum_{j \in \sI_2} \E[\tilde \lambda_i \tilde \lambda_j \tilde y_i \tilde y_j \vx_i^\top \rmQ \vx \vx_j^\top \rmQ \vx \mid x_i \in \gC_2, x_j \in \gC_2 ] 
\end{align*}
Let us develop each sum.
\paragraph{First sum}
We need to distinguish two cases here: case $i = j$ and $i \neq j$\\
- \textbf{For $ i \neq j:$}
\begin{align*}
    \E[\tilde \lambda_i \tilde \lambda_j \tilde y_i \tilde y_j \vx_i^\top \rmQ \vx \vx_j^\top \rmQ \vx \mid x_i \in \gC_1, x_j \in \gC_1] 
    &=  \E[\tilde \lambda_i \tilde \lambda_j \tilde y_i \tilde y_j \vx_i^\top \rmQ \vx \vx_j^\top \rmQ \vx \mid y_i = -1, y_j = -1] \\
    &= (\lambda_-^2 (1 - \varepsilon_-)^2  - 2 \lambda_- \lambda_+ \varepsilon_-(1 - \varepsilon_-) + \lambda_+^2 \varepsilon_-^2) \E[\vx_i^\top \rmQ \vx \vx_j^\top \rmQ \vx ] \\
    &= (\lambda_- (1 - \varepsilon_-) - \lambda_+ \varepsilon_-)^2 \E[\vx_i^\top \rmQ \vx \vx_j^\top \rmQ \vx ]\\
    &= (2\beta \varepsilon_- - \lambda_-)^2 \E[\vx_i^\top \rmQ \vx \vx_j^\top \rmQ \vx ] 
\end{align*}

We have that, knowing $x_i \in \gC_1, x_j \in \gC_1$ and $i \neq j$
\begin{align*}
    \E[\vx_i^\top \rmQ \vx \vx_j^\top \rmQ \vx] &= \E[\vx_i^\top \rmQ \vx \vx^\top \rmQ \vx_j] \\
    &= \E[x_i^\top \rmQ \E[\vx \vx^\top] \rmQ \vx_j] \quad \quad \quad (\vx \indep (\vx_i)_{i=1}^n) \\
    &= \E[\vx_i^\top \rmQ \Sigma_a \rmQ \vx_j] \\
    &= \frac{1}{(1 + \delta_1)^2} \E[\vx_i^\top \rmQ_{-i} \Sigma_a \rmQ_{-j} \vx_j] \\
    &= \frac{1}{(1 + \delta_1)^2} \E \left[\vx_i^\top \left ( \rmQ_{-ij} - \frac{\frac{1}{n} \rmQ_{-ij} \vx_j \vx_j^\top  \rmQ_{-ij} }{1 + \delta_1} \right) \Sigma_a \left ( \rmQ_{-ij} - \frac{\frac{1}{n} \rmQ_{-ij} \vx_i \vx_i^\top  \rmQ_{-ij} }{1 + \delta_1} \right) \vx_j \right] \\
    &= A_1 - A_2 - A_3 + A_4
\end{align*}

Let us compute each term now.
\begin{align*}
    A_1 &= \frac{1}{(1 + \delta_1)^2} \E[\vx_i\top \rmQ_{-ij} \Sigma_a \rmQ_{-ij} \vx_j] \\
    &= \frac{1}{(1 + \delta_1)^2} \vmu^\top \E[\rmQ_{-ij} \Sigma_a \rmQ_{-ij}] \vmu \\
    &= \frac{1}{(1 + \delta_1)^2} \vmu^\top \E[\rmQ \Sigma_a \rmQ] \vmu 
\end{align*}
Hence
\begin{equation}
    A_1 = \frac{1}{(1 + \delta_1)^2} \vmu^\top \E[\rmQ \Sigma_a \rmQ] \vmu
\end{equation}

And we have that:
\begin{align*}
    A_2 &= \frac{1}{(1 + \delta_1)^3} \E[\frac{1}{n}\vx_i^\top \rmQ_{-ij} \Sigma_a \rmQ_{-ij} \vx_i \vx_i^\top \rmQ_{-ij} \vx_j] \\
    &= \frac{1}{(1 + \delta_1)^3} \frac{1}{n} \Tr (\Sigma_1 \E[\rmQ \Sigma_a \rmQ]) \E[\vx_i^\top \rmQ_{-ij} \vx_j] \\
    &= \frac{1}{(1 + \delta_1)^3} \frac{1}{n} \Tr (\Sigma_1 \E[\rmQ \Sigma_a \rmQ]) \vmu^\top \bar \rmQ \vmu
\end{align*}
Since: 
\begin{align*}
    \frac{1}{n} \vx_i\top \rmQ_{-ij} \Sigma_a \rmQ_{-ij} \vx_i &= \frac{1}{n} \E[\vx_i\top \rmQ_{-ij} \Sigma_a \rmQ_{-ij} \vx_i] \\
    &= \frac{1}{n} \E[\Tr (\vx_i \vx_i\top \rmQ_{-ij} \Sigma_a \rmQ_{-ij})] \\
    &= \frac{1}{n} \Tr (\E[\vx_i \vx_i\top \rmQ_{-ij} \Sigma_a \rmQ_{-ij}]) \\
    &= \frac{1}{n} \Tr (\Sigma_1 \E[\rmQ \Sigma_a \rmQ])
\end{align*}
And we have that:
\begin{equation*}
    A_2 = A_3
\end{equation*}
And: 
\begin{equation*}
    A_4 = \gO(n^{-1})
\end{equation*}
Thus finally:
\begin{equation}
    \E[\vx_i^\top \rmQ \vx \vx_j^\top \rmQ \vx \mid  x_i \in \gC_1, x_j \in \gC_1 ] = \frac{1}{(1 + \delta_1)^2} \left ( \vmu^\top \E[\rmQ \Sigma_a \rmQ] \vmu - \frac{2}{(1 + \delta_1)} \frac{1}{n} \Tr (\Sigma_1 \E[\rmQ \Sigma_a \rmQ]) \vmu^\top \bar \rmQ \vmu  \right)
\end{equation}

Thus:
\begin{equation}
    \begin{split}
        &\E[\tilde \lambda_i \tilde \lambda_j \tilde y_i \tilde y_j \vx_i^\top \rmQ \vx \vx_j^\top \rmQ \vx \mid x_i \in \gC_1, x_j \in \gC_1] =  \frac{(2\beta \varepsilon_- - \lambda_-)^2}{(1 + \delta_1)^2} \\
        &\times\left ( \vmu^\top \E[\rmQ \Sigma_a \rmQ] \vmu - \frac{2}{(1 + \delta_1)} \frac{1}{n} \Tr (\Sigma_1 \E[\rmQ \Sigma_a \rmQ]) \vmu^\top \bar \rmQ \vmu  \right)
    \end{split}
\end{equation}

- \textbf{For $ i = j:$} we have that $\tilde y_i^2 = 1$ a.s, then knowing $\vx_i \in \gC_1$
\begin{align*}
    \E[\tilde \lambda_i^2 \tilde y_i^2 (\vx_i^\top \rmQ \vx)^2 ] 
    &= (\lambda_-^2 (1 - \varepsilon_-) + \lambda_+^2 \varepsilon_-)\E[(\vx_i^\top \rmQ \vx)^2 ] \\
    &= (4 \beta^2 \varepsilon_-(\rho_+ - \rho_-) + \lambda_-^2)\E[(\vx_i^\top \rmQ \vx)^2 ]
\end{align*}
And
\begin{align*}
    \E[ (\vx_i^\top \rmQ \vx)^2 ] &= \E[\vx_i^\top \rmQ \vx \vx^\top \rmQ \vx_i  ] \\
    &= \E[\vx_i^\top \rmQ \Sigma_a \rmQ \vx_i] \\
    &= \frac{1}{(1 + \delta_1)^2} \E[\Tr (\vx_i \vx_i^\top \rmQ_{-i} \Sigma_a \rmQ_{-i})] \\
    &= \frac{1}{(1 + \delta_1)^2} \Tr(\Sigma_1 \E[\rmQ \Sigma_a \rmQ]) 
\end{align*}
Thus:
\begin{equation}
    \E[\tilde \lambda_i^2 \tilde y_i^2 (\vx_i^\top \rmQ \vx)^2 \mid \vx_i \in \gC_1 ] = \frac{(4 \beta^2 \varepsilon_-(\rho_+ - \rho_-) + \lambda_-^2)}{ (1 + \delta_1)^2} \Tr(\Sigma_1 \E[\rmQ \Sigma_a \rmQ]) 
\end{equation}

\paragraph{Second sum:}
Here by definition, $i \neq j$. And we have, knowing $x_i \in \gC_1, x_j \in \gC_2$:
\begin{align*}
    &\E[\tilde \lambda_i \tilde \lambda_j \tilde y_i \tilde y_j \vx_i^\top \rmQ \vx \vx_j^\top \rmQ \vx \mid x_i \in \gC_1, x_j \in \gC_2] \\
    &= (\lambda_-^2 \varepsilon_+(1 - \varepsilon_-) - \lambda_+ \lambda_- (1 - \varepsilon_-)(1 - \varepsilon_+) - \lambda_+ \lambda_- \varepsilon_+ \varepsilon_- + \lambda_+^2 \varepsilon_-(1 - \varepsilon_+) ) \E[\vx_i^\top \rmQ \vx \vx_j^\top \rmQ \vx ] \\
    &= (2 \beta \varepsilon_+ - \lambda_+)(\lambda_- - 2 \beta \varepsilon_-) \E[\vx_i^\top \rmQ \vx \vx_j^\top \rmQ \vx ] 
\end{align*}
And, we have that:
\begin{align*}
    &\E[\vx_i^\top \rmQ \vx \vx_j^\top \rmQ \vx \mid x_i \in \gC_1, x_j \in \gC_2]  \\
    &= \E[\vx_i^\top \rmQ \Sigma_a \rmQ \vx_j] \\
    &= \frac{1}{(1 + \delta_1)(1 + \delta_2)} \E[\vx_i^\top \rmQ_{-i} \Sigma_a \rmQ_{-j} \vx_j] \\
    &= \frac{1}{(1 +\delta_1)(1 + \delta_2)} \E \left[\vx_i^\top \left ( \rmQ_{-ij} - \frac{\frac{1}{n} \rmQ_{-ij} \vx_j \vx_j^\top  \rmQ_{-ij} }{1 + \delta_2} \right) \Sigma_a \left ( \rmQ_{-ij} - \frac{\frac{1}{n} \rmQ_{-ij} \vx_i \vx_i^\top  \rmQ_{-ij} }{1 + \delta_1} \right) \vx_j \right] \\
    &= \frac{1}{(1 +\delta_1)(1 + \delta_2)} (B_1 - B_2 - B_3 + B_4)
\end{align*}
We have that:
\begin{align*}
    B_1 = \E[\vx_i^\top \rmQ_{-ij} \Sigma_a \rmQ_{-ij} \vx_j] = -\vmu^\top \E[\rmQ \Sigma_a \rmQ] \vmu
\end{align*}
And
\begin{align*}
    B_2 &= \frac{1}{n(1 + \delta_1)} \E[\vx_i^\top \rmQ_{-ij} \Sigma_a \rmQ_{-ij} \vx_i \vx_i^\top \rmQ_{-ij} \vx_j] \\
    &= \frac{1}{n(1 + \delta_1)} \Tr (\Sigma_1 \E[\rmQ \Sigma_a \rmQ]) \E[\vx_i^\top \rmQ_{-ij} \vx_j] \\
    &= \frac{-1}{n(1 + \delta_1)} \Tr (\Sigma_1 \E[\rmQ \Sigma_a \rmQ]) \vmu^\top \bar \rmQ \vmu
\end{align*}

And,
\begin{align*}
    B_3 &= \frac{1}{n(1 + \delta_2)} \E[\vx_i^\top \rmQ_{-ij} \vx_j \vx_j^\top \rmQ_{-ij} \Sigma_a \rmQ_{-ij} \vx_j] \\
    &= \frac{1}{n(1 + \delta_2)} \E[\vx_i^\top \rmQ_{-ij} \vx_j] \Tr(\Sigma_2 \E[\rmQ \Sigma_a \rmQ]) \\
    &= \frac{-1}{n(1 + \delta_2)} \Tr(\Sigma_2 \E[\rmQ \Sigma_a \rmQ]) \vmu^\top \bar \rmQ \vmu
\end{align*}
And $ B_4 = \gO(n^{-1})$\\
Thus, finally:
\begin{align*}
    &\E[\tilde \lambda_i \tilde \lambda_j \tilde y_i \tilde y_j \vx_i^\top \rmQ \vx \vx_j^\top \rmQ \vx \mid x_i \in \gC_1, x_j \in \gC_2] \\
    &= \frac{(\lambda_+ - 2 \beta \varepsilon_+)(\lambda_- - 2 \beta \varepsilon_-)}{(1 + \delta_1)(1 + \delta_2)} ( \vmu^\top \E[\rmQ \Sigma_a \rmQ] \vmu -  \frac{1}{n(1 + \delta_1)} \Tr (\Sigma_1 \E[\rmQ \Sigma_a \rmQ]) \vmu^\top \bar \rmQ \vmu \\
    &- \frac{1}{n(1 + \delta_2)} \Tr (\Sigma_2 \E[\rmQ \Sigma_a \rmQ]) \vmu^\top \bar \rmQ \vmu  )
\end{align*}

\paragraph{Third sum:} We have that \\
- \textbf{For $ i \neq j:$}
\begin{align*}
    \E[\tilde \lambda_i \tilde \lambda_j \tilde y_i \tilde y_j \vx_i^\top \rmQ \vx \vx_j^\top \rmQ \vx \mid x_i \in \gC_2, x_j \in \gC_2] 
    &=  \E[\tilde \lambda_i \tilde \lambda_j \tilde y_i \tilde y_j \vx_i^\top \rmQ \vx \vx_j^\top \rmQ \vx \mid y_i = 1, y_j = 1] \\
    &= (\lambda_-^2 \varepsilon_+^2  - 2 \lambda_- \lambda_+ \varepsilon_+(1 - \varepsilon_+) + \lambda_+^2 (1 - \varepsilon_+)^2) \E[\vx_i^\top \rmQ \vx \vx_j^\top \rmQ \vx ] \\
    &= (\lambda_- \varepsilon_+ - \lambda_+ (1 - \varepsilon_+))^2 \E[\vx_i^\top \rmQ \vx \vx_j^\top \rmQ \vx ]\\
    &= (2\beta \varepsilon_+ - \lambda_+)^2 \E[\vx_i^\top \rmQ \vx \vx_j^\top \rmQ \vx ]
\end{align*}
Thus:
\begin{equation}
    \E[\tilde \lambda_i \tilde \lambda_j \tilde y_i \tilde y_j \vx_i^\top \rmQ \vx \vx_j^\top \rmQ \vx \mid x_i \in \gC_2, x_j \in \gC_2] = \frac{(2\beta \varepsilon_+ - \lambda_+)^2}{(1 + \delta_2)^2} \left ( \vmu^\top \E[\rmQ \Sigma_a \rmQ] \vmu - \frac{2}{(1 + \delta_2)} \frac{1}{n} \Tr (\Sigma_2 \E[\rmQ \Sigma_a \rmQ]) \vmu^\top \bar \rmQ \vmu  \right)
\end{equation}
- \textbf{For $ i = j:$}
\begin{align*}
    \E[\tilde \lambda_i^2 \tilde y_i^2 (\vx_i^\top \rmQ \vx)^2 ] 
    &= (\lambda_-^2 \varepsilon_+ + \lambda_+^2 (1 - \varepsilon_+))\E[(\vx_i^\top \rmQ \vx)^2 ] \\
    &= (4 \beta^2 \varepsilon_+(\rho_- - \rho_+) + \lambda_+^2)\E[(\vx_i^\top \rmQ \vx)^2 ] \\
    &= \frac{(4 \beta^2 \varepsilon_+(\rho_- - \rho_+) + \lambda_+^2)}{(1 + \delta_2)^2}  \Tr(\Sigma_2 \E[\rmQ \Sigma_a \rmQ])
\end{align*}
Thus:
\begin{equation}
    \E[\tilde \lambda_i^2 \tilde y_i^2 (\vx_i^\top \rmQ \vx)^2 \mid \vx_i \in \gC_2 ] = \frac{(4 \beta^2 \varepsilon_+(\rho_- - \rho_+) + \lambda_+^2)}{(1 + \delta_2)^2}  \Tr(\Sigma_2 \E[\rmQ \Sigma_a \rmQ])
\end{equation}

Recall that we denoted by $T_1 = \frac1n \Tr(\Sigma_1 \E[\rmQ \Sigma_a \rmQ])$ and $T_2 = \frac1n \Tr(\Sigma_2 \E[\rmQ \Sigma_a \rmQ])$, we then deduce that:
\paragraph{Grouping all the terms:}
\begin{align*}
    \E[ g_{\rho}(\vx)^2] &= \frac{(\pi_1(\lambda_- - 2 \beta \varepsilon_-))^2}{(1 + \delta_1)^2} \left( \vmu^\top \E[\rmQ \Sigma_a \rmQ] \vmu - \frac{2}{1 + \delta_1} T_1 \vmu^\top \bar \rmQ \vmu \right) \\
    & + \frac{\pi_1 (4 \beta^2 \varepsilon_- (\rho_+ - \rho_-) + \lambda_-^2)}{(1 + \delta_1)^2} T_1 \\
    & + \frac{\pi_1 \pi_2(\lambda_+ - 2 \beta \varepsilon_+)(\lambda_- - 2 \beta \varepsilon_-)}{(1 + \delta_1)(1 + \delta_2)} \left( \vmu^\top \E[\rmQ \Sigma_a \rmQ] \vmu - \frac{1}{1 + \delta_1} T_1 \vmu^\top \bar \rmQ \vmu - \frac{1}{1 + \delta_2} T_2 \vmu^\top \bar \rmQ \vmu \right) \\
    & + \frac{(\pi_2(\lambda_+ - 2 \beta \varepsilon_+))^2}{(1 + \delta_2)^2} \left( \vmu^\top \E[\rmQ \Sigma_a \rmQ] \vmu - \frac{2}{1 + \delta_2} T_2 \vmu^\top \bar \rmQ \vmu \right) \\
    & + \frac{\pi_2 (4 \beta^2 \varepsilon_+ (\rho_- - \rho_+) + \lambda_+^2)}{(1 + \delta_2)^2} T_2 \\
    &= \left( \frac{\pi_1(\lambda_- - 2 \beta \varepsilon_-)}{1 + \delta_1} + \frac{\pi_2(\lambda_+ - 2 \beta \varepsilon_+)}{1 + \delta_2} \right)^2 \vmu^\top \E[\rmQ \Sigma_a \rmQ] \vmu \\
    &- \frac{T_1}{1 + \delta_1} \left( \left (\frac{\pi_1(\lambda_- - 2 \beta \varepsilon_-)}{1 + \delta_1} \right)^2 \vmu^\top \bar \rmQ \vmu + \frac{\pi_1 \pi_2(\lambda_+ - 2 \beta \varepsilon_+)(\lambda_- - 2 \beta \varepsilon_-)}{(1 + \delta_1)(1 + \delta_2)} \vmu^\top \bar \rmQ \vmu \right) \\
    & + \frac{\pi_1 (4 \beta^2 \varepsilon_- (\rho_+ - \rho_-) + \lambda_-^2)}{(1 + \delta_1)^2} T_1 + \frac{\pi_2 (4 \beta^2 \varepsilon_+ (\rho_- - \rho_+) + \lambda_+^2)}{(1 + \delta_2)^2} T_2 \\
    & - \frac{T_2}{1 + \delta_2} \left( \left (\frac{\pi_2(\lambda_+ - 2 \beta \varepsilon_+)}{1 + \delta_2} \right)^2 \vmu^\top \bar \rmQ \vmu + \frac{\pi_1 \pi_2(\lambda_+ - 2 \beta \varepsilon_+)(\lambda_- - 2 \beta \varepsilon_-)}{(1 + \delta_1)(1 + \delta_2)} \vmu^\top \bar \rmQ \vmu \right)
\end{align*}

\begin{remark}
    The expression $\E[\rmQ \Sigma_a \rmQ]$ can be easily inferred from this identity (obtained using lemma \ref{de QAQ}):
    \begin{equation}
        \E[\rmQ \Sigma_a \rmQ] = \bar \rmQ \Sigma_a \bar \rmQ + \frac{\pi_1}{n(1 + \delta_1)^2} \Tr(\Sigma_1 \bar \rmQ \Sigma_a \bar \rmQ) \E[\rmQ \Sigma_1 \rmQ] + \frac{\pi_2}{n(1 + \delta_2)^2} \Tr(\Sigma_2 \bar \rmQ \Sigma_a \bar \rmQ) \E[\rmQ \Sigma_2 \rmQ]
    \end{equation}
    So we get a system of two linear independent equations on $\E[\rmQ \Sigma_1 \rmQ]$ and $\E[\rmQ \Sigma_2 \rmQ]$, and therefore they are uniquely determined.
\end{remark}

\subsection{Isotropic Case}
If $\rmC = \rmI_p$, then we have that: 
\begin{align}
\label{trace-identity}
    &\delta_1 = \delta_2 = \delta , \quad & T_1 = T_2 = \frac1n \Tr((\Sigma \bar \rmQ)^2) = \frac{\eta (1 + \delta)^2}{(1 + \gamma (1 + \delta))^2}
\end{align}
and using lemma \ref{de QAQ}:
\begin{align}
    \E[\rmQ \Sigma \rmQ] = \frac{1}{1 - \frac1n \frac{\Tr((\Sigma \bar \rmQ)^2)}{(1 + \delta)^2}} \bar \rmQ \Sigma \bar \rmQ = \frac1h \bar \rmQ \Sigma \bar \rmQ
\end{align}
where : 
\begin{equation}
    h = 1 - \frac1n \frac{\Tr((\Sigma \bar \rmQ)^2)}{(1 + \delta)^2} = 1 - \frac{\eta}{(1 + \gamma (1 + \delta))^2}
\end{equation}

Hence, we get the following result:
\begin{corollary}[Gaussiannity of the label-perturbed classifier]
    Let $\vw_{\rho}$ be the LPC with parameters $\rho_\pm$, and $\bar \rmQ$ a deterministic equivalent of $\rmQ$ defined in lemma $\ref{DEQ}$. Under the same assumptions of \ref{assum_growth_rate}:
    \begin{align*}
    \vw_{\rho}^\top \vx \,\, \toind  \,\, \gN\left( (-1)^a m_\rho,\,  \nu_\rho - m_\rho^2 \right),
\end{align*}
where:
\begin{align*}
    &m_{\rho} = \frac{ \pi_1(\lambda_- - 2\beta \varepsilon_-) + \pi_2 (\lambda_+ - 2 \beta \varepsilon_+) }{1 + \delta} \vmu^\top \bar \rmQ \vmu, \\
    & \nu_\rho = \frac{(\pi_1(2 \beta \varepsilon_- - \lambda_-) + \pi_2(2 \beta \varepsilon_+ - \lambda_+))^2}{h (1 + \delta)^2} \left ( \mu^\top \bar \rmQ \Sigma \bar \rmQ \mu - \frac{2}{(1 + \delta)} \frac{1}{n } \Tr ((\Sigma \bar \rmQ)^2) \mu^\top \bar \rmQ \mu \right) \\
    &+ \frac{1}{h n(1 + \delta)^2} \pi_1 (4 \beta^2 \varepsilon_-(\rho_+ - \rho_-) + \lambda_-^2 ) \Tr ((\Sigma \bar \rmQ)^2) \\
    & + \frac{1}{h n(1 + \delta)^2} \pi_2 (4 \beta^2 \varepsilon_+(\rho_- - \rho_+) + \lambda_+^2 ) \Tr ((\Sigma \bar \rmQ)^2). \\
\end{align*}

\end{corollary}
We get Theorem \ref{thm_main} by simplifying further the expressions using lemma \ref{mu-identities} and \ref{trace-identity}:
\begin{align*}
    &m_{\rho} = \frac{ \pi_1(\lambda_- - 2\beta \varepsilon_-) + \pi_2 (\lambda_+ - 2 \beta \varepsilon_+) }{\Vert \vmu \Vert^2 + 1 + \gamma(1 + \delta)} \Vert \vmu \Vert^2, \\
    & \nu_\rho = \frac{(\pi_1(2 \beta \varepsilon_- - \lambda_-) + \pi_2(2 \beta \varepsilon_+ - \lambda_+))^2}{h (\Vert \vmu \Vert^2 + 1 + \gamma (1 + \delta))} \left ( \frac{\Vert \vmu \Vert^2 + 1}{\Vert \vmu \Vert^2 + 1 + \gamma (1 + \delta)} - 2(1 - h) \right) \Vert \vmu \Vert^2 \\
    &+ \frac{(1 - h)}{h} \left( \pi_1 (4 \beta^2 \varepsilon_-(\rho_+ - \rho_-) + \lambda_-^2 ) + \pi_2(4 \beta^2 \varepsilon_+(\rho_- - \rho_+) + \lambda_+^2) \right).
\end{align*}


\section{Additional plots}\label{appendix_plots}

Figure \ref{fig:accuracy-gamma} shows a consistent estimation of the test accuracy of different LPC variants as predicted by Proposition \ref{prop:test-accuracy}. 

\begin{figure}[H]
    \centering
    \includegraphics[width = \textwidth]{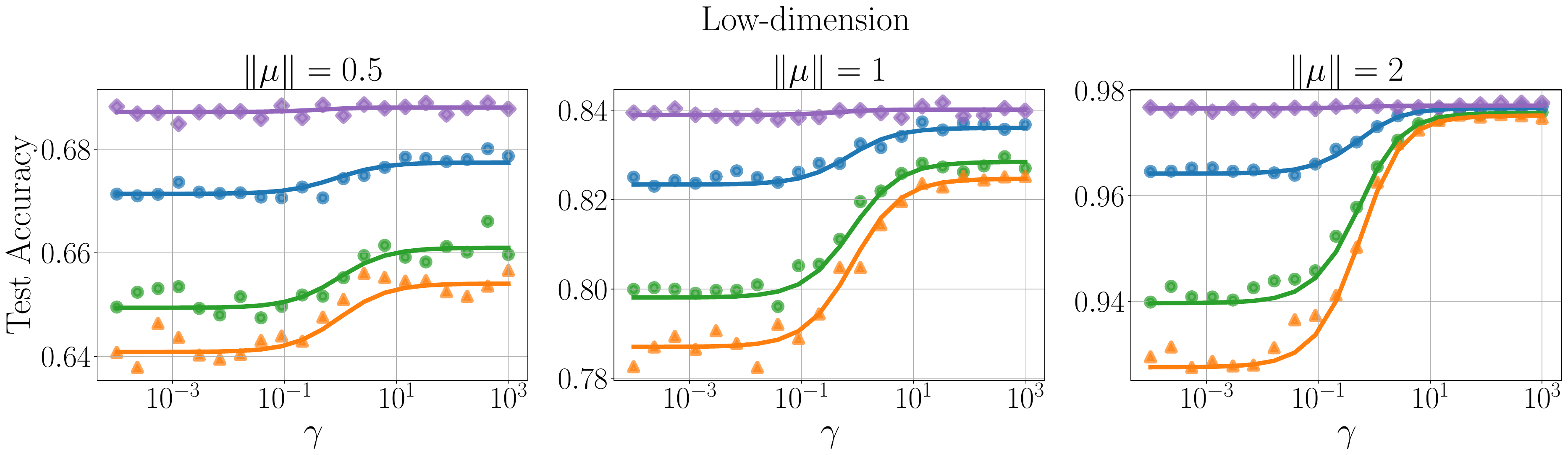}
    \includegraphics[width = \textwidth]{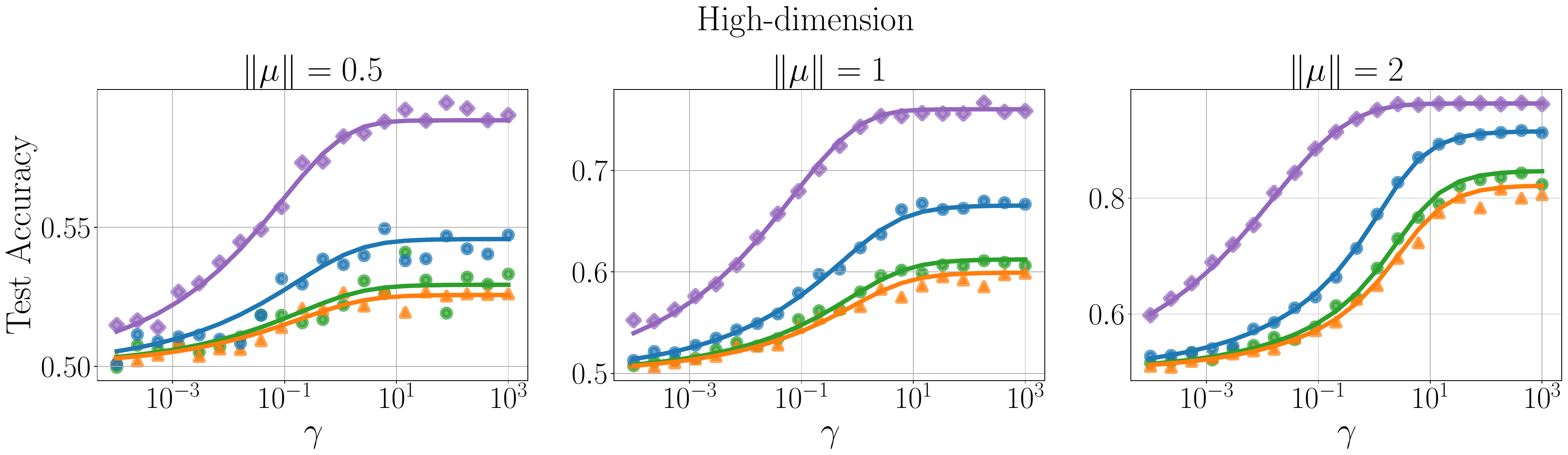}
    \includegraphics[width = \textwidth]{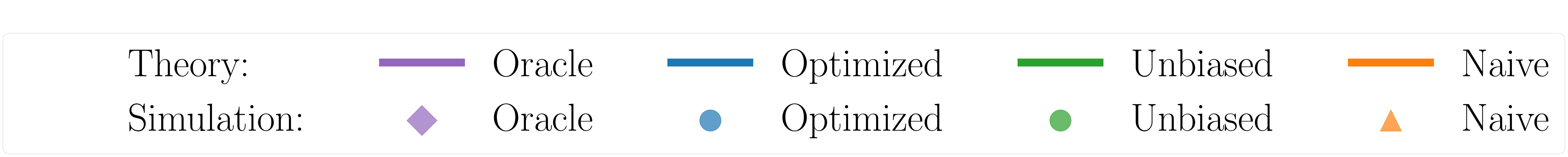}
    \caption{Empirical versus theoretical test accuracy as per Proposition \ref{prop:test-accuracy} for different variants of LPC. We used ($n, p = 2000, 20$) for Low-dimensional plot ($n, p = 200, 200$) and for High-dimensional experiment, $\pi_1 = 0.3$, $\varepsilon_+ = 0.4$, $\varepsilon_- = 0.3$ and varied $\gamma$.} 
    \label{fig:accuracy-gamma}
\end{figure}

Figure \ref{fig:epsilon-estimation} shows the result of estimating $\varepsilon_+$ using our approach as described in Section \ref{sec_noise_estimation}. We particularly notice that the estimated value of $\varepsilon_+$ is consistent even for small SNR $\Vert \vmu \Vert$.
\begin{figure}[H]
\centering
\includegraphics[width=\textwidth]{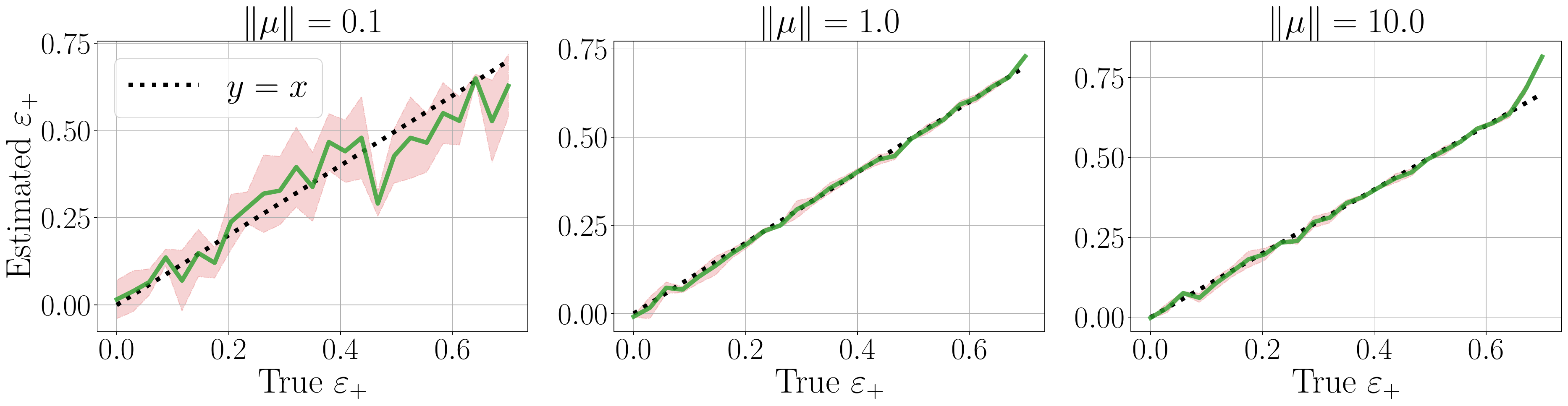}
\caption{ Estimation of the label noise rates as described in Section \ref{sec_noise_estimation}. We used $n = 1000$, $p = 100$, $\pi_1 = \frac13$, $\varepsilon_- = 0.2$, $(\rho_+^{(1)}, \rho_-^{(1)}) = (0, 0.1) $ and $(\rho_+^{(2)}, \rho_-^{(2)}) = (0, 0.4) $.}
\label{fig:epsilon-estimation}
\end{figure}

\section{Finding optimal parameters}\label{appendix_maximal_rho}
We denote by $\pi = \pi_1$ the proportion of data belonging to $\gC_1$ (hence $\pi_2 = 1 - \pi$). Our goal is to maximize the theoretical test accuracy as defined in \ref{prop:test-accuracy} with respect to $\rho_+$ for any fixed $\rho_-$. This is equivalent to maximizing the term  $\frac{m_\rho^2}{\nu_\rho - m_\rho^2}$ since $\varphi(\cdot)$ is a decreasing function. We have that:
\begin{align*}
    r(\rho_+) = \frac{m_\rho^2}{\nu_\rho - m_\rho^2} = \frac{N_1(\rho_+)}{D_1(\rho_+)}
\end{align*}
where:
\begin{equation*}
    N_1(\rho_+) = - h m_{\text{oracle}}^{2} \left(\pi \left(2 \epsilon_{-} + \rho_{+} - \rho_{-} - 1\right) - \left(\pi - 1\right) \left(2 \epsilon_{+} - \rho_{+} + \rho_{-} - 1\right)\right)^{2} \left(\rho_{+} + \rho_{-} - 1\right)^{2}
\end{equation*}
and 
\begin{align*}
    &D_1(\rho_+) = - h \left(\kappa - m_{\text{oracle}}^{2}\right) \left(\pi \left(2 \epsilon_{-} + \rho_{+} - \rho_{-} - 1\right) - \left(\pi - 1\right) \left(2 \epsilon_{+} - \rho_{+} + \rho_{-} - 1\right)\right)^{2} \left(\rho_{+} + \rho_{-} - 1\right)^{2} \\
    &+ \left(h - 1\right) \left(\pi \left(4 \epsilon_{-} \left(\rho_{+} - \rho_{-}\right) + \left(- \rho_{+} + \rho_{-} + 1\right)^{2}\right) + \left(\pi - 1\right) \left(4 \epsilon_{+} \left(\rho_{+} - \rho_{-}\right) - \left(\rho_{+} - \rho_{-} + 1\right)^{2}\right)\right) \left(\rho_{+} + \rho_{-} - 1\right)^{2}
\end{align*}
And differentiating $r$ with respect to $\rho_+$ gives us:
\begin{align*}
    r'(\rho_+) = \frac{N_2(\rho_+)}{D_2(\rho_+)}
\end{align*}

where :
\begin{align*}
    &N_2(\rho_+) = 2 h m_{\text{oracle}}^{2} (\pi (2 \epsilon_{-} + \rho_{+} - \rho_{-} - 1) - (\pi - 1) (2 \epsilon_{+} - \rho_{+} + \rho_{-} - 1))\\
    &\times (- (\pi (2 \epsilon_{-} + \rho_{+} - \rho_{-} - 1) - (\pi - 1) (2 \epsilon_{+} - \rho_{+} + \rho_{-} - 1)) \\
    &\times (h (\kappa - m_{\text{oracle}}^{2}) (2 \pi - 1) (\pi (2 \epsilon_{-} + \rho_{+} - \rho_{-} - 1) - (\pi - 1) (2 \epsilon_{+} - \rho_{+} + \rho_{-} - 1)) (\rho_{+} + \rho_{-} - 1) \\
    &+ h (\kappa - m_{\text{oracle}}^{2}) (\pi (2 \epsilon_{-} + \rho_{+} - \rho_{-} - 1) - (\pi - 1) (2 \epsilon_{+} - \rho_{+} + \rho_{-} - 1))^{2} \\
    &- (h - 1) (\pi (4 \epsilon_{-} (\rho_{+} - \rho_{-}) + (- \rho_{+} + \rho_{-} + 1)^{2}) + (\pi - 1) (4 \epsilon_{+} (\rho_{+} - \rho_{-}) - (\rho_{+} - \rho_{-} + 1)^{2})) \\
    &- (h - 1) (\pi (2 \epsilon_{-} + \rho_{+} - \rho_{-} - 1) + (\pi - 1) (2 \epsilon_{+} - \rho_{+} + \rho_{-} - 1)) (\rho_{+} + \rho_{-} - 1)) \\
    &+ (h (\kappa - m_{\text{oracle}}^{2}) (\pi (2 \epsilon_{-} + \rho_{+} - \rho_{-} - 1) - (\pi - 1) (2 \epsilon_{+} - \rho_{+} + \rho_{-} - 1))^{2} \\
    &- (h - 1) (\pi (4 \epsilon_{-} (\rho_{+} - \rho_{-}) + (- \rho_{+} + \rho_{-} + 1)^{2}) + (\pi - 1) (4 \epsilon_{+} (\rho_{+} - \rho_{-}) - (\rho_{+} - \rho_{-} + 1)^{2}))) \\
    &\times (\pi (2 \epsilon_{-} + \rho_{+} - \rho_{-} - 1) - (\pi - 1) (2 \epsilon_{+} - \rho_{+} + \rho_{-} - 1) + (2 \pi - 1) (\rho_{+} + \rho_{-} - 1)))
\end{align*}
And finally, solving $N_2(\rho_+) = 0$ gives us two solutions:
\begin{align*}
    &\rho_+^* = \frac{\pi^2 \epsilon_-(\epsilon_- - 1) + (1 - \pi)^2 \epsilon_+(1 - \epsilon_+)}{\pi (1 - \pi)(1 - \epsilon_+ - \epsilon_-)} + \rho_-, \quad & \bar \rho_+ = \frac{1 - 2 \pi \varepsilon_- - 2 (1 - \pi) \varepsilon_+}{2 \pi - 1} + \rho_-.
\end{align*}

\section{Loss Generalization}\label{appendix_bce_loss}
To investigate the extension of our approach to other bounded losses in addition to the squared loss considered in the main paper, we evaluated our LPC trained with the label perturbed loss (\ref{lpc-loss}) using a binary-cross-entropy loss, that is:
\begin{equation}
    \ell(s(\vx), y) = -y \log \left (s(\vx) \right) - (1 - y) \log \left (1 - s(\vx) \right),
\end{equation}
where $s(\vx) = \frac{1}{1 + \exp(-\vw^\top \vx)}$ and $y$ is in $\{ 0, 1\}$. Figures \ref{fig:bce-loss-synthetic} and \ref{fig:bce-loss-real} summarize the obtained test accuracies by setting $\rho_-$ to zero and varying $\rho_+$ on both synthetic and real data respectively. As anticipated theoretically with the squared loss, we remark similar behavior about the existence of an optimal $\rho_+^*$ that maximizes the accuracy beyond the \textit{unbiased} approach.
\begin{figure}[H]
    \centering
    \includegraphics[width = 0.9\textwidth]{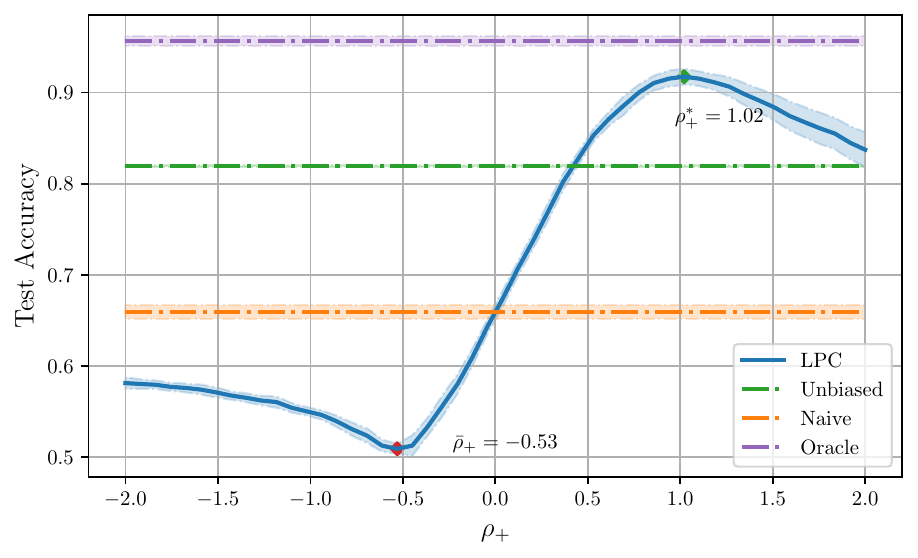}
    \caption{Test Accuracy on Synthetic data with classifiers obtained through minimizing the binary-cross-entropy loss using gradient descent. We used the parameters $n = 1000$, $p = 1000$, $\pi_1 = 0.3$, $\Vert \vmu \Vert = 2$, $\varepsilon_+ = 0.4$, $\varepsilon_- = 0.3$ and a learning rate of $0.1$.}
    \label{fig:bce-loss-synthetic}
\end{figure}

\begin{figure}[H]
    \centering
    \includegraphics[width = 0.9\textwidth]{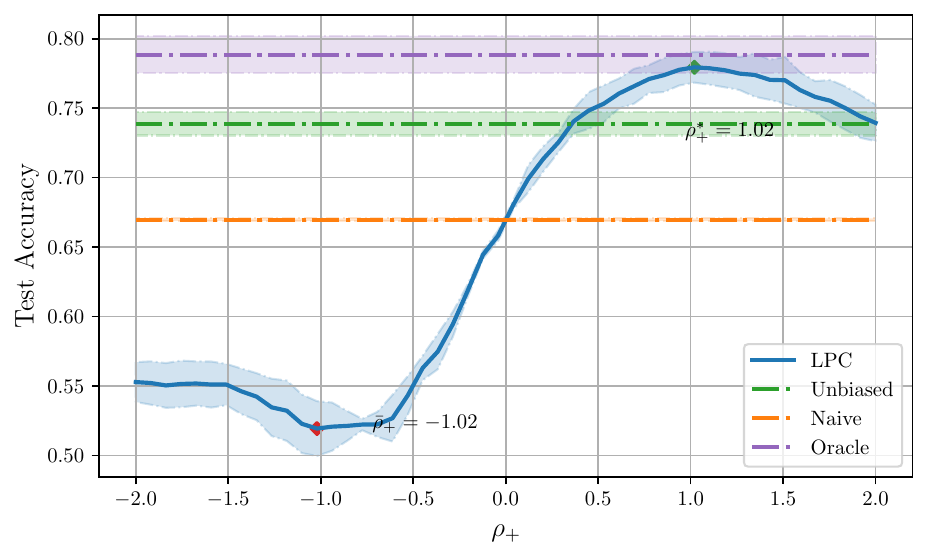}
    \caption{Test Accuracy on \texttt{Dvd} Amazon dataset \citep{blitzer2007biographies} with classifiers obtained through minimizing the binary-cross-entropy loss using gradient descent. We used the parameters $n = 1600$, $p = 400$, $\pi_1 = 0.3$, $\Vert \vmu \Vert = 2$, $\varepsilon_+ = 0.3$, $\varepsilon_- = 0.2$ and a learning rate of $0.1$.}
    \label{fig:bce-loss-real}
\end{figure}

\section{Multi-class extension: Multi-LPC}\label{appendix_multi_class}
In this section, we provide some evidence to show that our setting can be further extended to multi-class classification by considering the following settings.
\subsection{Setting}
We consider having a set of $n$ i.i.d $p$-dimensional vectors $\vx_1, \vx_2, ..., \vx_n \in \sR^p$ and corresponding labels $y_1, y_2, ..., y_n \in \{1, ..., k \}$ such that the $\vx_i$'s are sampled from a Gaussian mixture of $k$ clusters $\gC_1, ..., \gC_k$ with, $a \in \{1, ..., k \}$:
\begin{align*}
    &\vx_i \in \gC_a \quad \Leftrightarrow \quad \vx_i = \vmu_a + \vz_i,
\end{align*}
where $\vmu_a \in \sR^p$ and $\vz_i \in \gN(0, \rmI_p)$. We consider that the true labels are flipped randomly to get $\tilde y_1, \tilde y_2, ..., \tilde y_n$ such that for $a, b \in \{ 1, ..., k\}$:
\begin{align}
    \mathbb P(\tilde y_i = a  \mid  y_i = b) = \varepsilon_{a, b}, \quad \sum_{b= 1}^k \varepsilon_{a, b} < 1.
\end{align}

\subsection{Linear model}
Let $\vy_i \in \sR^k$ denote the one-hot encoding of the label $y_i$, i.e., if $\vx_i \in \gC_a$:
\begin{align*}
    \vy_{i, j} = \begin{cases}
        1 \quad \text{if} \quad j = a,\\
        0 \quad \text{otherwise}.
    \end{cases} 
\end{align*}
 Denote the data matrix $\rmX = [\vx_1, ..., \vx_n] \in \sR^{p \times n}$ and labels matrix $\rmY = [\vy_1, ..., \vy_n] \in \sR^{k \times n}$.
 \paragraph{Naive approach:}
 We consider a linear model that consists of minimizing the following regularized squared loss:
 \begin{equation}
 \label{naive-multi}
     \gL(\rmW) = \frac1n \sum_{i = 1}^n \Vert \tilde \vy_i - \rmW^\top \vx_i \Vert + \gamma \Vert \rmW \Vert^2_F,
 \end{equation}

where $\gamma \geq 0$ is a regularization parameter, and $\Vert . \Vert_F$ denotes the Frobenius norm of a matrix. The minimizer of this equation reads explicitly as:
\begin{align}
    \rmW = \frac1n \rmQ(\gamma) \rmX \tilde \rmY^\top,\quad \rmQ(\gamma) = \left( \frac1n \rmX \rmX^\top + \gamma \rmI_p \right)^{-1}.
\end{align}
\paragraph{Multi-LPC :} Let us sort the data vectors $(\vx_i)_{i = 1}^n$ in $\rmX$ and their labels $ (\tilde \vy_i)_{i = 1}^n$ in their matrices $\rmX$ and $\tilde \rmY$ such that we put the vectors of class $\gC_1$ in the first columns, then those of class $\gC_2$, and so on.\\
Let $\tilde \rmY^\top = [\vu_1, ..., \vu_k]$, each vector $\vu_i$ is defined in the following way:
\begin{align}
    \vu_{i, j} = \begin{cases}
        1 \quad \text{if} \quad \sum_{a=1}^{i - 1} \tilde n_a \leq j < \sum_{a=1}^{i} \tilde n_a \\
        0 \quad \text{otherwise}
    \end{cases} 
\end{align}
where $\tilde n_a$ is the number of noisy samples belonging to class $\gC_a$, i.e., the cardinality of this set $\{ i \in \{1, ..., n\}  \mid \tilde y_i = a\}$.
Now let $\alpha_1, ..., \alpha_k, \beta_1, ..., \beta_k \in \sR$. Our Multi-LPC approach consists of considering the following label matrix:
\begin{align}
    \rmY_{\alpha, \beta}^\top &= [\alpha_1 \vu_1 + \beta_1 (\mathbf {1}_n - \vu_1), ..., \alpha_k \vu_k + \beta_k (\mathbf {1}_n - \vu_k)] \\
    & = \tilde \rmY^\top \rmD(\alpha) + ( \rmM_1 - \tilde \rmY^\top) \rmD(\beta) 
\end{align}
where $\rmM_1 \in \sR^{n \times k} $ is the matrix containing $1$ in all its entries, and $\rmD(\alpha) \in \sR^{k \times k}$ (resp. $, \rmD(\beta) \in \sR^{k \times k}$) is a diagonal matrix containing the coefficients $\alpha_1, ..., \alpha_k$ (resp. $\beta_1, ..., \beta_k$) in its diagonal. Thus the multi-class LPC classifier is defined as:
\begin{equation}
    \rmW = \frac1n \rmQ(\gamma) \rmX \tilde \rmY_{\alpha, \beta}^\top.
\end{equation}
Our aim is to show the existence of parameters $(\alpha_i^*)_{i = 1}^k$ and $(\beta_i^*)_{i = 1}^k$ that maximize the accuracy of the classifier. 
\begin{remark}
    Remark that we can recover the Naive classifier in (\ref{naive-multi}) by taking $\alpha_i = 1$ and $\beta_i = 0$ for all $i \in \{ 1, ..., k\}$.
\end{remark}

\subsection{Experiments}
We tested our extension for $k=3$ and $k=4$ classes using synthetic data by taking:
\paragraph{For 3 classes ($k = 3$):}
We considered the following noise parameters matrix $\bm{\varepsilon}$ and the proportions $\bm{\pi}$ of data in each class ($\bm{\pi}_i$ is the proportion of data belonging to class $\gC_i$):
\begin{align*}
    &\bm{\varepsilon} = 
    \begin{pmatrix}
    0 & 0.3 & 0 \\
    0 & 0 & 0.4 \\
    0.5 & 0 & 0
    \end{pmatrix}
    & \bm{\pi} = \left ( 0.3, 0.3, 0.4\right)
\end{align*}
We also considered class $\gC_3$ of mean vector $\vmu_3$ of norm $\Vert \vmu_3 \Vert = 2$, class $\gC_1$ of mean $\vmu_1 = -\vmu_3$ and a centered class $\gC_2$ (zero norm mean).
\paragraph{For 4 classes ($k = 4$):}
We considered the parameters:
\begin{align*}
    &\bm{\varepsilon} = 
    \begin{pmatrix}
    0 & 0 & 0.5 & 0 \\
    0 & 0 & 0 & 0.3 \\
    0 & 0.4 & 0 & 0 \\
    0.3 & 0 & 0 & 0
    \end{pmatrix}
    & \bm{\pi} = \left ( 0.3, 0.2, 0.3, 0.2 \right)
\end{align*}
We also considered classes $\gC_3$ and $\gC_4$ of mean vectors $\vmu_3$ and $\vmu_4$ respectively such that: $\Vert \vmu_3 \Vert = 2$ and $\Vert \vmu_3 \Vert = 6$, and considered $\gC_1$ of mean $\vmu_1 = -\vmu_4$ and $\gC_2$ of mean $\vmu_2 = - \vmu_3$.\\

For each number of classes $k$, we found the optimal parameters (in terms of accuracy) $\bm{\alpha^*} = (\alpha^*_i)_{i = 1}^k$ and $\bm{\beta^*} = (\beta^*_i)_{i = 1}^k$ and also the worst ones $\bm{\bar \alpha} = (\bar \alpha_i)_{i = 1}^k$ and $\bm{\bar \beta} = (\bar \beta)_{i = 1}^k$ within a grid of $ G = 5000 $ parameters, using Monte Carlo simulation. To visualize the results,
we report the accuracy of the Multi-LPC approach with the parameters $\bm{\alpha}_\tau = \tau \bm{\alpha^*} + (1 - \tau) \bm{\bar \alpha} $ and  $\bm{\beta}_\tau = \tau \bm{\beta^*} + (1 - \tau) \bm{\bar \beta} $ by varying the parameter $\tau \in (0, 1)$. Figure \ref{fig:multi-class} summarizes the obtained results and we clearly observe improved accuracy for $(\bm{\alpha^*}, \bm{\beta^*})$ even approaching the oracle classifier.
\begin{figure}[H]
    \centering
    \includegraphics[width = \textwidth]{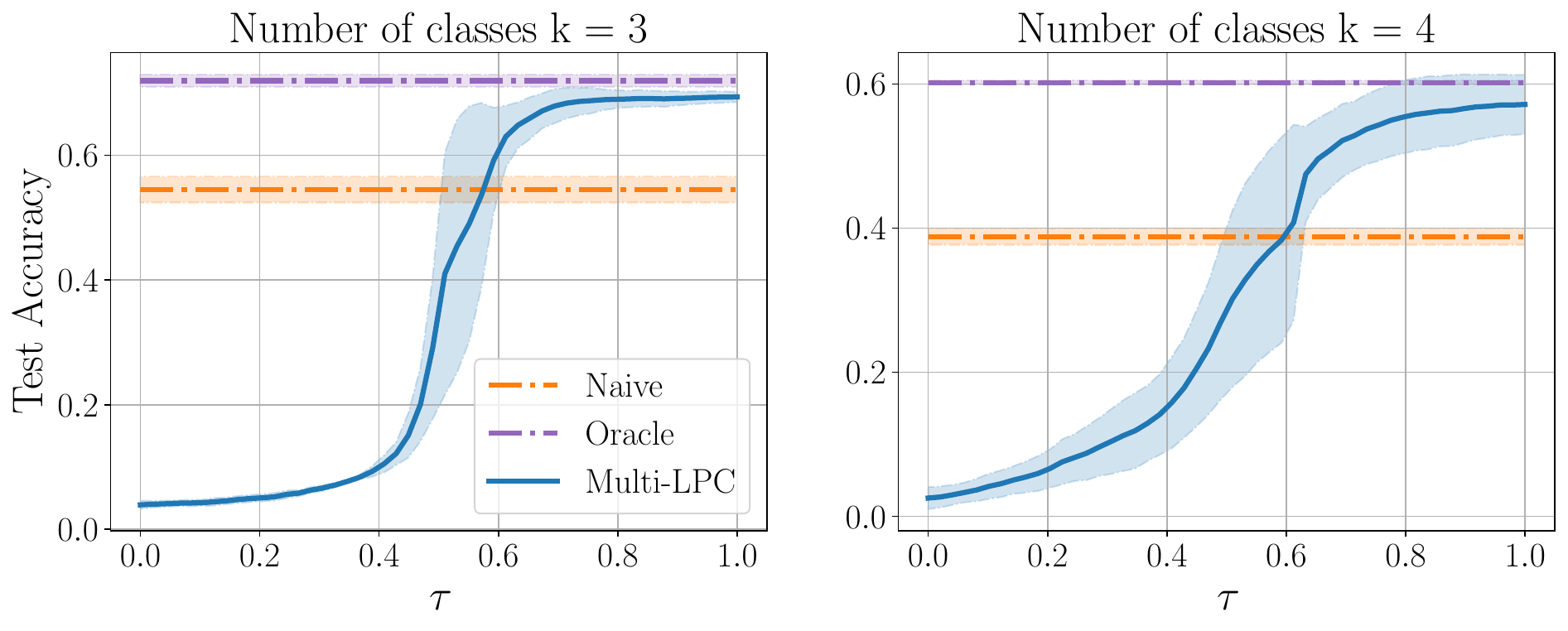}
    \caption{Multi-class classification with $n = 2000$, $p = 200$ evaluated on $3$ random seeds.}
    \label{fig:multi-class}
\end{figure}

\end{document}